%% file: iclr2027_conference.tex
\documentclass{article}
\usepackage{inconsolata}
\usepackage{iclr2027_conference,times}
\iclrfinalcopy 
\input{math_commands.tex}

\usepackage{hyperref}
\usepackage{cleveref}
\usepackage{url}
\usepackage{graphicx}
\usepackage{booktabs,tabularx,array,caption}
\usepackage{xcolor}
\usepackage{algorithm,algpseudocode}
\usepackage{tcolorbox}
\tcbuselibrary{breakable,skins}

\usepackage{amsmath}
\usepackage{amssymb}
\usepackage{mathtools}
\usepackage{amsthm}
\usepackage{multirow}
\usepackage{graphicx}
\usepackage{listings}
\usepackage{enumitem}
\usepackage{wrapfig}
\usepackage{textcomp}
\usepackage{microtype}
\usepackage{inconsolata}

\usepackage{url}            
\usepackage{booktabs}       
\usepackage{amsfonts}       
\usepackage{nicefrac}       
\usepackage{microtype}      
\usepackage{xcolor}         

\usepackage{xcolor}
\usepackage{bm}       
\usepackage{enumitem}
\usepackage{bbm} 
\usepackage{listings}

\title{From Pareto to Preference: Personalized Test-Time Scaling via Amortized Agentic Policy Discovery}

\author {
    \textbf{Xinglin Wang}\textsuperscript{\rm 1}\footnotemark[1], \hspace{0cm}
    \textbf{Zishen Liu}\textsuperscript{\rm 1}\footnotemark[1], \hspace{0cm} 
    \textbf{Tong Zheng}\textsuperscript{}\footnotemark[1], \hspace{0cm} 
    \textbf{Shaoxiong Feng}\textsuperscript{\rm 2}\footnotemark[2], \hspace{0cm}
    \textbf{Peiwen Yuan}\textsuperscript{\rm 1}, \hspace{0cm}
    \textbf{Yiwei Li}\textsuperscript{\rm 1}, \hspace{0cm}
    \\
    \hspace{0.05cm} \textbf{Jiayi Shi}\textsuperscript{\rm 1}, \hspace{0cm}
    \textbf{Yueqi Zhang}\textsuperscript{\rm 1}, \hspace{0cm}
    \textbf{Chuyi Tan}\textsuperscript{\rm 1}, \hspace{0cm}
    \textbf{Ji Zhang}\textsuperscript{\rm 1}, \hspace{0cm}
    \textbf{Boyuan Pan}\textsuperscript{\rm 2}, \hspace{0cm}  
    \textbf{Kan Li}\textsuperscript{\rm 1}\footnotemark[2] \\
    \hspace{0.05cm} \textsuperscript{\rm 1} Beijing Institute of Technology  \quad \textsuperscript{\rm 2} Xiaohongshu Inc \\
    \hspace{0.05cm} \texttt{\{wangxinglin,liuzishen,peiwenyuan,liyiwei\}@bit.edu.cn} \\
    \hspace{0.05cm} \texttt{\{shijiayi,zhangyq,tanchuyi, zhangji, likan\}@bit.edu.cn} \\
    \hspace{0.05cm} \texttt{\{zhengtong12356,shaoxiongfeng2023\}@gmail.com}\\ \hspace{0.05cm} \texttt{panboyuan@xiaohongshu.com}
}
\usepackage{amsthm}

\newtheorem{pttsproposition}{Proposition}

\begin{document}

\maketitle
\lhead{Preprint}

\ificlrfinal
\begingroup
\renewcommand{\thefootnote}{\fnsymbol{footnote}}
\footnotetext[1]{Equal contribution.}
\footnotetext[2]{Corresponding author.}
\endgroup
\fi

\input{abstract}

\input{introduction}

\input{related_work}

\input{method}
\input{experiment}
\input{conclusion}
\input{ai_statement}
\input{ethic}
\input{reproduc}

\bibliography{iclr2027_conference}
\bibliographystyle{iclr2027_conference}

\clearpage
\input{appendix}

\end{document}

%% file: math_commands.tex
\usepackage{amsmath,amsfonts,bm}

\def\eqref#1{equation~\ref{#1}}

\def\1{\bm{1}}

\DeclareMathAlphabet{\mathsfit}{\encodingdefault}{\sfdefault}{m}{sl}
\SetMathAlphabet{\mathsfit}{bold}{\encodingdefault}{\sfdefault}{bx}{n}



%% file: abstract.tex
\begin{abstract}
Test-time scaling (TTS) improves the reasoning capabilities of large language models by allocating additional inference computation.
Existing approaches to improving TTS efficiency largely optimize accuracy against one resource dimension at a time, advancing either the \emph{accuracy--cost} or \emph{accuracy--latency} Pareto frontier.
Yet user requirements are multidimensional: users may specify accuracy, latency, and inference-cost requirements jointly, and different requirements can favor different controllers.
We formulate \emph{Personalized Test-Time Scaling} as discovering executable controllers that maximize the joint satisfaction rate of user-specific requirements.
To reduce the overhead of repeated policy discovery for new user profiles, we propose \textbf{PersonTTS}, an amortized agentic policy-discovery framework that reuses prior search experience through requirement-matched controller initialization and source-distilled procedural guidance, while retaining target-profile evaluation for every candidate.
Experiments on AIME and HMMT show that PersonTTS substantially outperforms strong TTS baselines in joint requirement satisfaction on unseen user profiles and held-out problems.
Under the same candidate-evaluation budget, cross-user experience reuse further improves policy quality while substantially reducing discovery-agent time and cost\footnote{Our code and data have been released on \url{https://github.com/WangXinglin/PersonTTS}.}.
\end{abstract}

%% file: introduction.tex
\section{Introduction}
\label{sec:introduction}

\begin{figure*}[t]
    \centering
    \includegraphics[width=\textwidth]{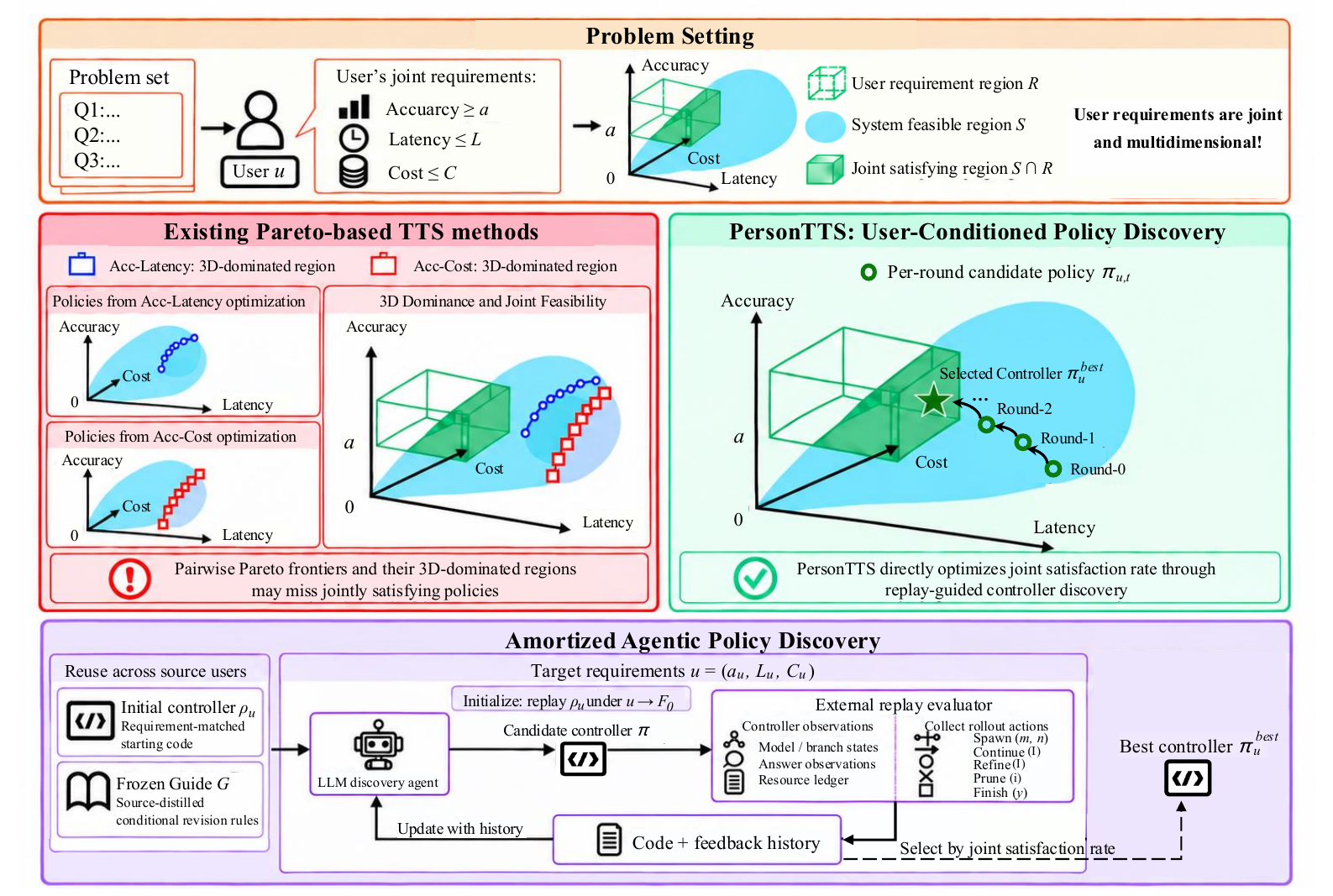}
    \caption{
    \textbf{Comparison of Existing Test-Time Scaling Paradigms and Personalized Test-Time Scaling.}
    Existing methods primarily optimize accuracy against latency or inference cost separately, whereas Personalized TTS targets their joint satisfaction under user-specific requirements.
    PersonTTS amortizes user-conditioned controller discovery by reusing requirement-matched controllers and procedural guidance from prior searches.
    }
    \label{fig:personalized-discovery}
    \vspace{-2pt}
\end{figure*}

Test-time scaling (TTS) has emerged as a powerful paradigm for improving the reasoning capabilities of large language models~\citep{o1, Kimi-k1.5,DeepSeek-R1}.
By allocating additional computation at inference time, TTS enables extended exploration of the solution space, allowing models to explore diverse reasoning paths, evaluate candidate answers, and refine solutions for complex tasks such as mathematical reasoning~\citep{SC,Self-Refine,lightman2024let}.
However, scaling up this exploration incurs substantial inference overhead~\citep{brown2024large,snell2025scaling}.
To mitigate these overheads, existing approaches largely focus on advancing the \emph{accuracy--cost} or \emph{accuracy--latency} Pareto frontier by reducing redundant computation~\citep{wu2025inference,wang2025faster,wang2026every, wang2026share,zheng2026autotts}, for example, using adaptive sampling, early stopping, or selective branch pruning~\citep{ASC,ESC,wang2025make,zheng2026parallelprobe}.

While yielding promising efficiency gains, these approaches typically optimize accuracy against one resource dimension at a time, which does not fully capture a user's joint requirements.
In practice, users may require a target accuracy subject to both latency and inference-cost limits~\citep{huang2025latency,wang2026time}, so different requirements can favor different controllers even on the same problem set.
Consequently, optimizing either Pareto frontier can fail to identify controllers that satisfy accuracy, latency, and inference-cost requirements jointly (Figure~\ref{fig:personalized-discovery}).
We therefore formulate \emph{Personalized Test-Time Scaling} as discovering executable controllers that maximize the joint satisfaction rate of user-specific accuracy, latency, and inference-cost requirements, shifting the focus from Pareto to preference.

Crucially, personalization affects not only how a controller is evaluated, but also how it should execute.
An executable TTS controller coordinates model selection, reasoning width and depth, refinement, pruning, and stopping~\citep{zheng2026autotts}, and these decisions can depend on the target requirements.
Candidate controllers must therefore be evaluated under the target profile rather than assumed to transfer across users.
However, running policy discovery independently for every new profile can be costly, e.g., \$39.9 and 160 minutes for a complete policy-discovery run~\citep{zheng2026autotts}.

To address this challenge, we propose \textbf{PersonTTS}, an amortized agentic policy-discovery framework that reuses prior search experience across users.
Specifically, PersonTTS retrieves a controller from a source profile with similar requirements and re-evaluates it under the target profile to initialize discovery, while a frozen Guide distilled from source discovery histories informs subsequent proposals (see Figure~\ref{fig:personalized-discovery}).
An LLM discovery agent then proposes executable controllers using the target requirements, replay feedback, and accumulated search history.
Each candidate is evaluated under the target profile, and the controller with the highest joint satisfaction rate is retained.
By combining controller reuse and procedural guidance with target-specific evaluation, PersonTTS amortizes design effort across users and enables more efficient personalized policy discovery.

We evaluate PersonTTS on AIME and HMMT~\citep{dekoninck2026matharena} using six Qwen3 models~\citep{qwen3technicalreport}, with independently sampled source and target requirement profiles and separate held-out problem sets.
PersonTTS substantially outperforms strong TTS baselines in joint requirement satisfaction on both unseen user profiles and held-out problems.
Under the same candidate-evaluation budget, cross-user experience reuse further improves policy quality over independent target-side discovery while substantially reducing discovery-agent time and monetary cost.
Ablation studies further characterize the distinct roles of requirement-matched initialization and procedural guidance in personalized policy discovery.

Our contributions are summarized as follows:
\begin{itemize}
    \item We formulate \textbf{Personalized Test-Time Scaling} as executable controller discovery under joint user-specific accuracy, latency, and inference-cost requirements, with the objective of maximizing their joint satisfaction.

    \item We propose \textbf{PersonTTS}, an amortized agentic policy-discovery framework that reuses requirement-matched controllers and procedural discovery experience across users while preserving target-specific evaluation.

    \item We empirically validate PersonTTS on AIME and HMMT, demonstrating substantial gains over strong TTS baselines on unseen user profiles and held-out problems, while cross-user experience reuse further improves policy quality and reduces discovery overhead under the same candidate-evaluation budget.
\end{itemize}

%% file: related_work.tex
\section{Related Work}
\label{sec:related-work}

\paragraph{Efficient Test-Time Scaling.}
Test-time scaling improves reasoning by allocating additional inference computation through parallel sampling and aggregation~\citep{SC,brown2024large,wu2025inference,snell2025scaling}, iterative feedback and refinement~\citep{Self-Refine}, or structured search over intermediate reasoning states~\citep{ToT,besta2024got}.
A growing body of work seeks to make this additional computation more efficient by reducing unnecessary exploration.
Adaptive sampling and early stopping terminate generation once sufficient evidence has been accumulated~\citep{ASC,ESC}, while methods conditioned on problem difficulty or reasoning quality allocate computation adaptively across problems or reasoning paths~\citep{wan2025rasc,wang2025make,snell2025scaling,wang2026every}.
Other approaches improve the organization of inference through branch pruning, adaptive control of search width and depth, parallel execution, or cross-branch information sharing~\citep{wang2025faster,huang2025latency,wang2026rse,wang2026share,zheng2026parallelprobe,zheng2026autotts}.
Despite their different mechanisms, these methods can largely be viewed as advancing the \emph{accuracy--cost} or \emph{accuracy--latency} Pareto frontier.
In contrast, Personalized TTS considers user-specific accuracy, latency, and inference-cost requirements jointly, and seeks executable controllers that maximize their joint satisfaction rather than optimizing a single accuracy--resource Pareto frontier.

\paragraph{Agentic Discovery.}
Large language models are increasingly used as discovery agents that iteratively propose, evaluate, and refine executable artifacts, ranging from algorithms and programs to agentic workflows~\citep{funsearch,eoh,novikov2025alphaevolve,hu2025automated,zhang2025aflow}.
Recent approaches further leverage textual feedback or execution histories to guide subsequent proposals, enabling more targeted optimization of prompts, modules, and system code~\citep{agrawal2026gepa,lee2026meta}.
AutoTTS brings this paradigm to test-time scaling by formulating TTS strategy design as executable controller synthesis in an offline replay environment, targeting the accuracy--cost Pareto frontier~\citep{zheng2026autotts}.
PersonTTS extends this paradigm to user-specific joint accuracy, latency, and inference-cost requirements, while amortizing the otherwise repeated controller discovery across profiles through reusable prior experience.

\paragraph{Experience Reuse in LLMs.}
Prior work has explored reusing past experience to improve subsequent behavior, including reflections and distilled insights~\citep{shinn2023reflexion,zhao2024expel}, reasoning and search experience distilled from successful and failed trajectories~\citep{ouyang2026reasoningbank,wang2026rse}, reusable executable skills~\citep{wang2023voyager} and replay world simulators~\citep{zheng2026dreamrsirecursiveselfimprovementevolving}.
At the design level, SWIFT distills reusable design knowledge from previous workflow searches, while FlowBank precomputes and selects workflows for new queries~\citep{du2026swift,yuan2026flowbank}. 
PersonTTS instead focuses on reuse across user requirement profiles, using source controllers to initialize target discovery and distilled procedural experience to guide subsequent proposals, while retaining target-profile evaluation for controller selection.

%% file: method.tex
\section{Methodology}
\label{sec:method}

We formulate Personalized Test-Time Scaling as a user-conditioned
controller-discovery problem:
\begin{quote}
\emph{Given a user's accuracy, latency, and inference-cost requirements,
how can we efficiently automate personalized TTS policy design to maximize
the probability of jointly satisfying these requirements?}
\end{quote}
To address this problem, we propose PersonTTS, which combines
feedback-driven program search with cross-user experience reuse.
For each target requirement, retrieval provides an initial policy, a frozen
Guide informs candidate revisions, and replay-based selection retains the
evaluated policy with the highest joint satisfaction rate.

\subsection{Problem Formulation}
\label{sec:problem_formulation}

We represent a user's requirements by a profile
$u=(a_u,L_u,C_u)$, which specifies an accuracy floor $a_u\in[0,1]$,
a ceiling $L_u$ on maximum per-question replay latency, and a ceiling
$C_u$ on mean per-question inference cost.
In an environment $e$ with models $\mathcal M_e$, a policy $\pi$ is executable code that maps $u$ and the current observation history $h_k$ to an action. Public observations include model and branch availability, branch progress, intermediate answers and answer histories, and accumulated latency and cost. As correctness is assessed externally, the controller receives neither reference answers nor correctness labels.

\paragraph{Action space.}
A policy controls \emph{model routing}, \emph{reasoning width},
\emph{reasoning depth}, and \emph{self-refinement}. Calling a reasoning path
a branch, we write $\pi(u,h_k)\in\mathcal A_e$, where
\begin{equation}
\label{eq:action_space}
\mathcal A_e=\{\operatorname{Spawn}(m,n),\operatorname{Continue}(I),
\operatorname{Refine}(I),\operatorname{Prune}(i),\operatorname{Finish}(y)\}.
\end{equation}
Here, $\operatorname{Spawn}(m,n)$ opens $n$ branches with model
$m\in\mathcal M_e$, $\operatorname{Continue}(I)$ advances the selected
branch set $I$, and $\operatorname{Refine}(I)$ starts self-refinement for
eligible branches whose current stage has ended. $\operatorname{Prune}(i)$
excludes branch $i$ from further reasoning and answer selection.
$\operatorname{Finish}(y)$ ends the question with an observed answer $y$,
or deterministic plurality voting if $y$ is omitted. The environment
determines action admissibility, while the full interface is given in
Appendix~\ref{app:public_api}. Program search can therefore revise both
resource allocation and the rules that adapt it to intermediate observations.

\paragraph{Joint satisfaction objective.}
Fix a labeled question set $\mathcal D=\{(q_i,y_i^\star)\}_{i=1}^{N}$,
$N\ge1$, a cache $\mathcal R_e$, and a nonempty finite replay-seed panel
$\Omega$. Each seed specifies a complete evaluation over $\mathcal D$,
called a \emph{seed batch}, and varies cached-branch consumption order.
For complete finite replay records
$\tau_u^\pi(s)=\operatorname{Replay}_e(\pi,u;\mathcal R_e,\mathcal D,s)$, define
\begin{equation}
\label{eq:replay_metrics}
A_u^\pi(s)=\frac1N\sum_i\chi_i(\tau_u^\pi(s)),\quad
L_u^\pi(s)=\max_{q\in\mathcal D}\ell_s(q;\pi,u),\quad
C_u^\pi(s)=\frac1N\sum_i c_i(\tau_u^\pi(s)),
\end{equation}
where $\chi_i$ checks the returned answer against $y_i^\star$, $c_i$ is
the inference cost for question $q_i$, and $\ell_s(q;\pi,u)$ is its replay
latency, with $L_u^\pi(s)=\Lambda_e(\tau_u^\pi(s))$.
We therefore define the \emph{joint satisfaction rate} (JSR) as the
fraction of seed batches satisfying all three requirements simultaneously:
\begin{equation}
\label{eq:jsr}
\widehat S_u(\pi;\mathcal D,\Omega)=\frac1{|\Omega|}\sum_{s\in\Omega}
\mathbf1\!\left[A_u^\pi(s)\ge a_u,\ L_u^\pi(s)\le L_u,\ C_u^\pi(s)\le C_u\right].
\end{equation}
JSR measures the fraction of seed batches meeting all three requirements simultaneously.
As changing $u$ affects both the acceptance thresholds and potentially the policy's actions, candidates must be evaluated under the target profile.

\subsection{User-Conditioned Agentic Policy Discovery}
\label{sec:persontts}

The search space includes branching, stopping, and refinement logic as well
as model and sampling choices.
To efficiently explore this space for policies satisfying a user's requirements, an LLM agent revises executable controller logic while retaining AutoTTS's offline replay foundation and within-search reuse of code and diagnostics~\citep{zheng2026autotts}.
Cached trajectories, intermediate answers, and resource records let us evaluate candidates without additional rollouts.
The controller acts on observations revealed by replay within each query, while the agent revises its code using the resulting feedback between evaluations.

A scalar JSR ranks candidates but cannot distinguish an accuracy shortfall
from a resource violation. Each evaluation therefore returns
$F_t=(X_t,\zeta_t,\mathcal T_t)$: the JSR $X_t$, constraint pass rates and
margins in $\zeta_t$, and sanitized executed traces $\mathcal T_t$ linking
these outcomes to controller decisions. Appendix~\ref{app:diagnostic_limits}
characterizes what exact constraint pass rates and JSR reveal about
single-constraint failures.
Starting from $\rho_u$, we obtain initial feedback $F_0$ and then generate
$B$ new candidates:
\begin{equation}
\label{eq:candidate_generation}
\pi_{u,t}\sim p_\theta(\cdot\mid e,u,\rho_u,\mathcal H_{u,<t},
\widehat\eta_{u,t},G),\qquad t=1,\ldots,B.
\end{equation}
where $\mathcal H_{u,<t}$ contains accessible prior code and feedback,
$\widehat\eta_{u,t}$ contains trace-based latency and token-cost estimates,
and $G$ is the optional Guide (Section~\ref{sec:experience_reuse}).
The initial code $\rho_u$ is a retrieved source policy when warm-start is enabled and a generic template otherwise.
It stays fixed across rounds, while the history accumulates all scored candidates, including unsuccessful revisions.
Adaptation changes program code, not the agent parameters $\theta$.

The agent inspects evidence and submits one candidate per round for
external static validation and replay.
However, it does not execute or score candidates itself.
The full procedure is given in Algorithm~\ref{alg:persontts} in
Appendix~\ref{app:discovery_algorithm}, where the discovery set, seed panel,
and replay realizations are fixed.
Only a strict JSR improvement replaces the incumbent, so ties retain the earlier policy. Initialization
is evaluated separately from the $B$ new scored candidates, while agent time
and expenditure are measured separately.

\subsection{Cross-User Experience Reuse}
\label{sec:experience_reuse}

Offline replay removes repeated rollouts, but each new user still requires program design and diagnosis.
To reduce this repetition, we retain source search histories in a \emph{Policy Experience Bank}, including user profiles, candidate code, evaluation feedback, execution traces, and selected policies.
These records support two forms of reuse: retrieval provides a starting policy, while a Guide distilled from candidate comparisons informs subsequent revisions under the new user's requirements.

\paragraph{Requirement-similarity policy initialization.}
Let $\mathcal I_e$ index source entries with compatible models and runtime
configurations, with $\pi_i^{\mathrm{src}}$ selected for source profile $u_i$.
To compare resource ratios alongside accuracy differences, we log-transform
resource ceilings and standardize using compatible source profiles:
\begin{equation}
\label{eq:requirement_transform}
\mathbf x(u)=(a_u,\ln L_u,\ln C_u),\qquad
z_{e,j}(u)=\frac{x_j(u)-\mu_{e,j}}{\sigma_{e,j}},\quad j=1,2,3,
\end{equation}
where $\mu_{e,j}$ and $\sigma_{e,j}$ are source-coordinate means and standard deviations.
Retrieval assumes a nonempty compatible bank, positive resource ceilings, and positive coordinate standard deviations.
Equal-weight Euclidean distance then selects the initial program:
\begin{equation}
\label{eq:nearest_neighbor}
i^\star(u)\in\operatorname*{arg\,min}_{i\in\mathcal I_e}
\|\mathbf z_e(u)-\mathbf z_e(u_i)\|_2,\qquad
\rho_u=\pi_{i^\star(u)}^{\mathrm{src}}.
\end{equation}
We replay the retrieved code under the target profile to obtain $F_0$.
This establishes its target joint satisfaction rate (JSR) and supplies constraint diagnostics for
the first revision, rather than relying on its source score.

\paragraph{Procedural skill distillation and guided discovery.}
The retrieved policy and its initial feedback provide a starting point,
but deciding how to revise it requires evidence about alternative designs.
We therefore compare source candidates under their associated user
requirements, linking code changes to evaluation feedback and execution
traces.
An agent distills these comparisons into a shared Guide whose rules
specify when a revision is appropriate, which changes to prioritize, and
how to evaluate the resulting policy.
Supporting and unsuccessful cases help identify the conditions under which
each rule applies.

When enabled, the frozen Guide enters every proposal through
Equation~\ref{eq:candidate_generation}, while target feedback updates the
agent's search evidence rather than the Guide itself.
This allows source-side preparation to be shared across users while keeping
candidate generation and final selection tied to each target profile.
Construction details are provided in Appendices~\ref{app:guide_generation}
and~\ref{App:training_trace_methodology}.

\subsection{Separating Requirement Changes from Execution Changes}
\label{sec:reuse_transfer}

Reusing a source policy changes both how its outcomes are judged and how
it may execute. To separate these effects, fix source and target profiles
$u,v$, a common environment, cache, question set, and finite nonempty panel
$\Omega$. For a policy $\pi$ with complete finite outcomes under both profiles,
we denote its joint satisfaction rate (JSR) by
$f_{z,\Omega}(\pi)=\widehat S_z(\pi;\mathcal D,\Omega)$.
Orient metrics so that larger is better:
$\mathbf w_z^\pi(s)=(A_z^\pi(s),-L_z^\pi(s),-C_z^\pi(s))$ and
$\mathbf b(z)=(a_z,-L_z,-C_z)$. For fixed positive unit scales $d_j$, define
\begin{equation}
\label{eq:reuse_rethreshold}
g^\pi(s)=\min_{1\le j\le3}\frac{w_{u,j}^\pi(s)-b_j(v)}{d_j},\qquad
\bar f_\pi=\frac1{|\Omega|}\sum_{s\in\Omega}\mathbf1\{g^\pi(s)\ge0\}.
\end{equation}
The scales normalize units. Applying target thresholds to fixed source
outcomes gives $\bar f_\pi$, so
\begin{equation}
\label{eq:reuse_transfer_decomposition}
f_{v,\Omega}(\pi)-f_{u,\Omega}(\pi)
=[\bar f_\pi-f_{u,\Omega}(\pi)]+[f_{v,\Omega}(\pi)-\bar f_\pi].
\end{equation}
The first term captures threshold changes on fixed records, while the second
captures the effect of requirement-conditioned execution.
This decomposition explains why source-profile performance can inform target
discovery without generally determining target-profile performance.
Appendix~\ref{app:reuse_analysis} extends this analysis to comparisons
between policies, including ranking reversals under changed requirements
and the insufficiency of marginal metrics for determining JSR.

\paragraph{Proposition: target satisfaction bounds.}
Suppose finite bounds $r^\pi(s)$ satisfy
$r^\pi(s)\ge\max_j|w_{v,j}^\pi(s)-w_{u,j}^\pi(s)|/d_j$ for every seed.
Then $F_\pi^-\le f_{v,\Omega}(\pi)\le F_\pi^+$, where
\begin{equation}
\label{eq:reuse_target_envelope}
F_\pi^-=\frac1{|\Omega|}\sum_{s\in\Omega}\mathbf1\{g^\pi(s)\ge r^\pi(s)\},
\qquad
F_\pi^+=\frac1{|\Omega|}\sum_{s\in\Omega}\mathbf1\{g^\pi(s)\ge-r^\pi(s)\}.
\end{equation}
The proof is in Appendix~\ref{app:reuse_transfer_proof}. When $r^\pi(s)=0$
for every seed, rethresholding determines target JSR exactly. More generally,
the bound depends on execution deviation, not profile distance alone.
Observed target deviations yield a post-replay diagnostic, while prediction
before replay requires independently justified deviation bounds.
Accordingly, PersonTTS uses source similarity and design experience to guide
policy discovery, while retaining target execution for candidate evaluation
and selection.

%% file: experiment.tex
\section{Experiments}


\subsection{Experimental Setup}
\label{sec:experimental_setup}

\paragraph{Benchmarks and profiles.}
We evaluate PersonTTS on AIME and HMMT~\citep{dekoninck2026matharena}
using six Qwen3 models with 0.6B, 1.7B, 4B, 8B, 14B, and 32B
parameters~\citep{qwen3technicalreport}.
AIME24--25 provides 60 discovery problems, with AIME26 (30 problems)
held out, while HMMT24 provides 30 discovery problems, with HMMT25
(30 problems) held out.
For each problem--model pair, the replay pool contains 128
pre-sampled, checkpointed reasoning trajectories with intermediate
answers, token counts for thinking, probing, and feedback, and latency
records, as described in Appendix~\ref{app:pool_build}.
Each benchmark uses 100 source profiles for the main comparison and
20 independently sampled target profiles for cross-user experiments.
We sample profiles by difficulty-stratified maximin coverage calibrated
on discovery data, using the procedure in
Appendix~\ref{app:profile_sampling}; the source and target sets contain
no duplicate threshold triples. The experience bank and Guide are
constructed only from source-profile discovery histories, excluding
target-profile records and held-out problem outcomes. Additionally, we provide the full evaluation protocol and discovery settings in Appendices~\ref{app:cost_model}
and~\ref{app:discovery_settings}, respectively.

\paragraph{Baselines.}
We compare against three baseline families.
\textbf{AutoTTS}~\citep{zheng2026autotts} optimizes a scalar
accuracy--cost trade-off, for which we report $\beta=0.5$ and
$\beta=1.0$ separately.
\textbf{Self-Consistency} includes ASC~\citep{ASC} and ESC~\citep{ESC},
also reported separately. Both use fixed model, width, depth, and
refinement configurations while allowing runtime stopping decisions.
\textbf{Parallel-Probe}~\citep{zheng2026parallelprobe} combines
intermediate-answer probing, branch pruning, and self-refinement.
Together, these baselines compare personalized policy discovery with
scalar-objective search and predefined inference strategies under the
same joint-satisfaction metric.

\subsection{Main Results}
\label{sec:main_results}

To evaluate whether directly optimizing for user-specific joint requirements provides an advantage over existing TTS strategies, we compare PersonTTS and its no-reuse variant against strong external baselines under the same joint-satisfaction metric.
Table~\ref{tab:main_results} indicates that personalized controller discovery substantially outperforms these baselines across both benchmarks and retains this advantage on held-out problems, while the no-reuse variant already preserves most of this advantage.
This suggests that the primary gain comes from optimizing executable
controllers for the joint user requirements themselves, rather than from cross-user experience reuse alone. Building on this stronger personalized objective, cross-user reuse further improves target-profile policies under the same candidate-evaluation budget,
with the improvement persisting on held-out problems, indicating that source experience helps steer discovery toward controller designs that transfer beyond the problems used during search.


\begin{table}[t]
\centering
\caption{JSR (\%; higher is better) across problem and profile splits.
Source and target profiles follow the Setup, and problem splits are labeled separately.
PersonTTS variants use the policy with the highest JSR on discovery problems across
initialization and all five rounds.
PersonTTS uses both reuse mechanisms. ``w/o reuse'' removes both the Guide and
warm-start, while the other ``w/o'' variants remove the named mechanism.
Dashes denote unreported results. Bold marks the
highest reported value in each column.}
\label{tab:main_results}
\small
\setlength{\tabcolsep}{2pt}
\renewcommand{\arraystretch}{1.15}
\begin{tabular*}{\linewidth}{@{\extracolsep{\fill}}l*{8}{c}@{}}
\toprule
\multirow[c]{3}{*}[-0.5ex]{\textbf{Method}} & \multicolumn{2}{c}{\textbf{AIME24-25}}
& \multicolumn{2}{c}{\textbf{AIME26}}
& \multicolumn{2}{c}{\textbf{HMMT24}}
& \multicolumn{2}{c}{\textbf{HMMT25}} \\
& \multicolumn{2}{c}{(Discovery)} & \multicolumn{2}{c}{(Held-out)}
& \multicolumn{2}{c}{(Discovery)} & \multicolumn{2}{c}{(Held-out)} \\
\cmidrule(lr){2-3}\cmidrule(lr){4-5}\cmidrule(lr){6-7}\cmidrule(lr){8-9}
& \multicolumn{1}{c}{Source} & \multicolumn{1}{c}{Target}
& \multicolumn{1}{c}{Source} & \multicolumn{1}{c}{Target}
& \multicolumn{1}{c}{Source} & \multicolumn{1}{c}{Target}
& \multicolumn{1}{c}{Source} & \multicolumn{1}{c}{Target} \\
\midrule
ASC & 13.53 & 0.14 & 7.03 & 0.00 & 11.91 & 4.89 & 12.03 & 4.98 \\
ESC & 13.28 & 0.14 & 6.22 & 0.00 & 11.22 & 4.89 & 11.83 & 4.98 \\
ParallelProbeSR & 23.77 & 10.19 & 23.89 & 9.35 & 23.81 & 19.30 & 18.07 & 11.51 \\
AutoTTS ($\beta=0.5$) & 12.11 & 16.41 & 10.52 & 16.03 & 7.04 & 2.96 & 7.25 & 3.23 \\
AutoTTS ($\beta=1.0$) & 31.15 & 35.58 & 30.11 & 35.47 & 0.00 & 0.01 & 0.00 & 0.15 \\
\midrule
PersonTTS (w/o reuse) & \textbf{88.35} & 85.57 & \textbf{78.60} & 79.22 & \textbf{91.41} & 87.04 & \textbf{85.71} & 69.16 \\
PersonTTS (w/o warm-start) & -- & 96.45 & -- & 83.50 & -- & 88.27 & -- & \textbf{79.09} \\
PersonTTS (w/o Guide) & -- & 95.13 & -- & 79.60 & -- & \textbf{91.19} & -- & 74.04 \\
PersonTTS & -- & \textbf{96.57} & -- & \textbf{83.85} & -- & 90.86 & -- & 77.28 \\
\bottomrule
\end{tabular*}
\end{table}

\subsection{Ablation Study}
\label{sec:ablation}

To disentangle how the two forms of cross-user experience reuse contribute to policy discovery, we separately remove requirement-matched initialization and the Guide under the same target profiles and candidate-evaluation budget. Figure~\ref{fig:round_jsr} and Table~\ref{tab:main_results} reveal distinct roles for the two mechanisms: warm-start mainly improves the initial search point, whereas the Guide informs subsequent revisions and provides more
consistent held-out gains across benchmarks. Their gains are not uniformly additive, suggesting that retrieved controllers act as target-dependent initialization priors, while procedural guidance is
less dependent on a particular source controller.

\begin{figure}[t]
    \centering
    \includegraphics[width=\linewidth]{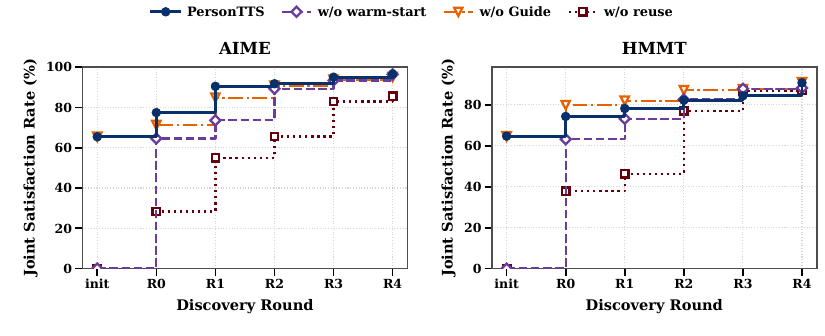}
    \caption{\textbf{Best-observed JSR on discovery sets across rounds.}
    Curves show the cumulative maximum JSR on the discovery sets after
    initialization and each of the five candidate rounds (R0--R4), averaged
    over target profiles on AIME24--25 and HMMT24. The four curves compare
    PersonTTS with the variants without warm-start, without the Guide,
    and without both mechanisms.}
    \label{fig:round_jsr}
\end{figure}

To determine whether the policy-quality gains from experience reuse are
accompanied by lower discovery overhead, we further compare agent-call time
and cost under the same five-round protocol.
Table~\ref{tab:agent_overhead} reveals that the Guide accounts for most
of the consistent efficiency improvement: relative to no reuse, Guide-only
discovery reduces agent-call time by approximately $46\%$ and cost by
approximately $36\%$ across the two benchmarks, whereas warm-start alone
mainly shortens elapsed time and can slightly increase cost.
Together with the policy-quality ablation, this pattern suggests that
distilled procedural experience reduces repeated diagnosis and revision
effort throughout discovery, while requirement-matched retrieval primarily
serves as an initialization prior whose value is more dependent on the
target setting.

\begin{table}[t]
    \small
    \centering
    \renewcommand{\arraystretch}{0.92}
    \setlength{\tabcolsep}{2.2pt}

    \newcommand{\timecost}[2]{%
        \makebox[2.6em][r]{#1}$\,|\,$\makebox[2.6em][l]{#2}%
    }

    \caption{
    Agent-call elapsed time and cost during target-profile discovery
    on discovery problems.
    Each cell reports time (minutes) $\,|\,$ cost (USD), with lower
    values better for both metrics.
    Values are averaged over 20 target profiles per benchmark and round.
    Total gives the per-profile sum across the five rounds.
    Variant names follow Table~\ref{tab:main_results}.
    }
    \label{tab:agent_overhead}

    \resizebox{\linewidth}{!}{
    \begin{tabular}{@{}l*{6}{c}@{}}
        \toprule
        \multirow[c]{2}{*}[-0.5ex]{\textbf{Method}}
        & \multicolumn{5}{c}{\textbf{Round}}
        & \multirow[c]{2}{*}[-0.5ex]{\textbf{Total}} \\
        \cmidrule(lr){2-6}
        & R0 & R1 & R2 & R3 & R4 & \\
        \midrule
        \multicolumn{7}{c}{\textbf{\textit{AIME}}} \\
        \midrule
        PersonTTS (w/o reuse)
            & \timecost{18.67}{1.58}
            & \timecost{17.37}{2.29}
            & \timecost{14.56}{2.20}
            & \timecost{12.33}{1.94}
            & \timecost{11.70}{2.29}
            & \timecost{74.63}{10.30} \\
        PersonTTS (w/o Guide)
            & \timecost{12.74}{1.61}
            & \timecost{13.72}{2.03}
            & \timecost{14.70}{2.36}
            & \timecost{12.85}{2.27}
            & \timecost{12.31}{2.37}
            & \timecost{66.33}{10.64} \\
        PersonTTS (w/o warm-start)
            & \timecost{13.60}{1.24}
            & \timecost{7.05}{1.29}
            & \timecost{7.56}{1.46}
            & \timecost{6.13}{1.28}
            & \timecost{5.35}{1.37}
            & \timecost{39.69}{6.64} \\
        PersonTTS
            & \timecost{8.49}{1.15}
            & \timecost{8.90}{1.51}
            & \timecost{8.49}{1.67}
            & \timecost{7.83}{1.57}
            & \timecost{6.70}{1.50}
            & \timecost{40.41}{7.40} \\
        \midrule
        \multicolumn{7}{c}{\textbf{\textit{HMMT}}} \\
        \midrule
        PersonTTS (w/o reuse)
            & \timecost{22.06}{2.05}
            & \timecost{18.63}{2.47}
            & \timecost{14.82}{2.40}
            & \timecost{15.80}{2.72}
            & \timecost{14.18}{2.70}
            & \timecost{85.49}{12.34} \\
        PersonTTS (w/o Guide)
            & \timecost{15.41}{2.25}
            & \timecost{16.14}{2.50}
            & \timecost{14.80}{2.46}
            & \timecost{17.12}{2.99}
            & \timecost{12.83}{2.56}
            & \timecost{76.30}{12.76} \\
        PersonTTS (w/o warm-start)
            & \timecost{14.05}{1.40}
            & \timecost{9.00}{1.54}
            & \timecost{8.29}{1.67}
            & \timecost{7.66}{1.63}
            & \timecost{7.29}{1.62}
            & \timecost{46.29}{7.87} \\
        PersonTTS
            & \timecost{11.68}{1.74}
            & \timecost{11.34}{1.80}
            & \timecost{10.40}{2.03}
            & \timecost{7.45}{1.61}
            & \timecost{7.63}{1.66}
            & \timecost{48.51}{8.84} \\
        \bottomrule
    \end{tabular}
    }
\end{table}
\subsection{Analysis}
\label{sec:analysis}

We analyze the proposed PersonTTS from the following perspectives:
(1) discovery-to-held-out generalization,
(2) experience bank scaling,
(3) cross-benchmark generalization (Appendix~\ref{sec:cross_benchmark}).
\subsubsection{Discovery-to-Held-out Generalization}
\label{sec:train_test_transfer}

To examine whether round-wise discovery JSR is informative of out-of-sample policy quality, we evaluate the policy produced at each discovery round on both the optimization and held-out problems. Figure~\ref{fig:rounds_train_test} shows that improvements in discovery JSR are generally accompanied by stronger held-out policies, with PersonTTS finishing above independent discovery on both benchmarks. The two trajectories are not perfectly aligned, however, and the persistent discovery--held-out gap indicates that discovery JSR is an informative search signal rather than a direct estimate of out-of-sample joint satisfaction.

\begin{figure}[t]
\centering
\includegraphics[width=\linewidth]{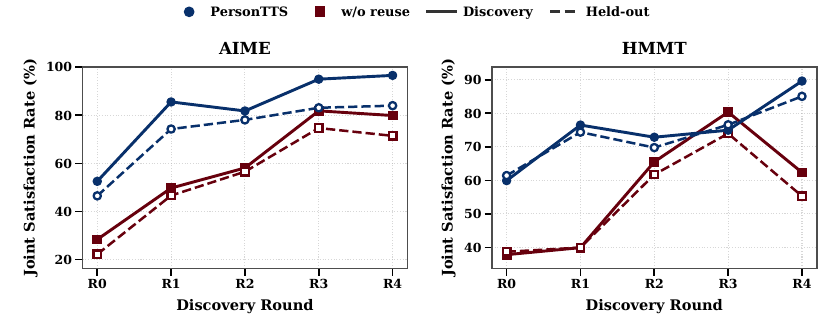}
\caption{\textbf{Discovery-to-held-out generalization across rounds.}
Solid and dashed lines report the JSR of the policy produced at each
discovery round on the discovery and held-out problems, respectively, for
PersonTTS and the variant without cross-user reuse on AIME and HMMT.
Held-out evaluations are used only for analysis and never exposed to the
discovery agent or used for policy selection.}
\label{fig:rounds_train_test}
\end{figure}

\subsubsection{Experience Bank Scaling}
\label{sec:bank_scaling}

To understand how the amount of reusable source experience affects target discovery, we vary the number of source-profile--policy pairs available for retrieval while keeping the Guide and target-side candidate-evaluation budget fixed. For each benchmark, five bank-sampling seeds construct nested banks of 20 and 60 source-profile--policy pairs from the full 100-pair bank, and each sampled bank is used in a separate discovery run for the same target profiles.

The results in Table~\ref{tab:bank_scaling} indicate that enlarging the bank consistently improves policy quality on the discovery problems, suggesting that broader coverage of requirement-specific controller designs increases the chance of retrieving a useful starting point. This trend does not extend monotonically to held-out problems, however: larger banks continue to help on AIME but reverse on HMMT, showing that broader retrieval coverage can facilitate target-side search without guaranteeing stronger cross-problem generalization. Because the Guide and candidate-evaluation budget remain fixed, this pattern isolates a limitation of retrieval-side scaling and motivates calibrated coverage or retrieval-confidence estimates rather than treating bank size as
a uniformly beneficial scaling axis.

\begin{table}[t]
\centering
\caption{Experience-bank scaling: JSR (\%; higher is better) of the final
discovery-selected policies. Scores are averaged over available bank-seed
runs within each target profile, then equally over 20 profiles.
Bold marks the highest reported value in each column. The 100-profile bank is the full source bank used in the main comparison.}
\label{tab:bank_scaling}
\small
\setlength{\tabcolsep}{5pt}
\renewcommand{\arraystretch}{1.13}
\begin{tabular*}{\linewidth}{@{\extracolsep{\fill}}lcccc@{}}
\toprule
\multirow[c]{2}{*}{\textbf{Bank Size}} & \textbf{AIME24--25} & \textbf{AIME26}
& \textbf{HMMT24} & \textbf{HMMT25} \\
& (Discovery) & (Held-out) & (Discovery) & (Held-out) \\
\midrule
20 & 92.24 & 80.30 & 86.22 & \textbf{86.09} \\
60 & 94.51 & 81.38 & 87.40 & 85.91 \\
100 & \textbf{96.57} & \textbf{83.85}
& \textbf{90.86} & 77.28 \\
\bottomrule
\end{tabular*}
\par\smallskip
\begin{minipage}{\linewidth}
\footnotesize
\end{minipage}
\end{table}

%% file: conclusion.tex
\section{Conclusions}

In this work, we formulate Personalized Test-Time Scaling, which seeks executable controllers that jointly satisfy user-specific accuracy, latency, and inference-cost requirements rather than optimizing a single accuracy--resource Pareto frontier.
We propose PersonTTS, an amortized agentic policy-discovery framework that combines target-conditioned controller search with cross-user reuse of requirement-matched source controllers and procedural guidance, while retaining target-profile evaluation for candidate selection.
Experiments on AIME and HMMT demonstrate substantial gains over strong TTS baselines on unseen user profiles and held-out problems, while cross-user experience reuse further improves policy quality and reduces discovery overhead under the same candidate-evaluation budget.
Overall, our results highlight the value of amortizing personalized TTS policy design across users.
Future work could combine calibrated coverage or uncertainty estimates with request-frequency-aware deployment decisions, reusing matched controllers when the experience bank is reliable and triggering background discovery for frequent or poorly covered profiles whose search cost can be amortized over future requests.

%% file: ai_statement.tex
\section*{AI Use Statement}
In this work, we used generative AI tools to polish the manuscript. We have not
used generative AI tools for other tasks with required disclosure, and the
remaining disclosure tasks are not applicable to this work. Additionally, we
used generative AI tools to improve the manuscript's readability. We have
reviewed the polished text. We take responsibility for the final content of
this work, including text, claims, or artifacts produced with the aid of
generative AI.

%% file: ethic.tex
\section*{Ethics Statement}
This study uses offline replay with simulated operational requirement profiles
and mathematical reasoning problems; it involves no human participants and does
not measure subjective satisfaction. The reported joint-satisfaction metric is
defined by accuracy and resource constraints in this replay setting. Deployment
would require validation on the intended workload and resource accounting.

%% file: reproduc.tex
\section*{Reproducibility Statement}
Section~\ref{sec:method} defines the controller interface, joint-satisfaction
objective, retrieval procedure, and candidate-selection loop. Section~\ref{sec:experimental_setup}
describes the problem and profile splits, replay evaluation protocol, and
discovery settings. The appendix documents replay-pool construction,
profile sampling, the discovery prompt and public API, Guide distillation,
round-wise measurements, and the proofs for the conditional transfer analysis.

%% file: appendix.tex
\appendix
\numberwithin{table}{section}

\section{User-Conditioned Policy Discovery Algorithm}
\label{app:discovery_algorithm}

\begin{algorithm}[H]
\caption{User-Conditioned Policy Discovery}
\label{alg:persontts}
\small
\begin{algorithmic}[1]
\Require Environment $e$, cache $\mathcal R_e$, requirements $u$, fixed discovery set $\mathcal D$
and seed panel $\Omega$; initial policy $\rho_u$ with a well-defined JSR;
optional Guide $G$; $B$ new scored candidates
\State $F_0\gets\Call{Evaluate}{\rho_u;e,\mathcal R_e,u,\mathcal D,\Omega}$
\State $(\pi_{\mathrm{best}},X_{\mathrm{best}})\gets(\rho_u,\Call{JSR}{F_0})$;
$\mathit{records}\gets[(\rho_u,F_0)]$
\For{$t=1,\ldots,B$}
    \State $\mathcal H_{u,<t}\gets$ accessible context from $\mathit{records}$
    \State $\widehat\eta_{u,t}\gets$ calibration estimates from available traces
    \State $\pi_{u,t}\gets\Call{AgentGenerate}{e,u,\rho_u,\mathcal H_{u,<t},\widehat\eta_{u,t},G}$
    \State \Call{StaticValidate}{$\pi_{u,t}$}
    \State $F_t\gets\Call{Evaluate}{\pi_{u,t};e,\mathcal R_e,u,\mathcal D,\Omega}$
    \State Append $(\pi_{u,t},F_t)$ to $\mathit{records}$
    \If{$\Call{JSR}{F_t}>X_{\mathrm{best}}$}
        \State $(\pi_{\mathrm{best}},X_{\mathrm{best}})\gets(\pi_{u,t},\Call{JSR}{F_t})$
    \EndIf
\EndFor
\State \Return $\pi_{\mathrm{best}}$
\end{algorithmic}
\end{algorithm}

\section{Implementation Details}
\label{app:implementation_details}

\subsection{Offline Replay Pool Construction}
\label{app:pool_build}

We construct a frozen replay pool from offline model rollouts. For each
problem--model pair in Section~\ref{sec:experimental_setup}, we collect
128 rollouts containing extracted answers, think/probe/feedback token
counts, and wall-clock latency records. Each rollout consists of an
initial reasoning stage followed by up to 5 self-refinement stages.
Within each stage, we place a checkpoint every 500 thinking tokens and
use an answer probe of approximately 30 tokens to extract the current
answer.

Rollout generation uses temperature $T=0.6$, top-$p=0.95$, and top-$k=20$.
Correctness is checked by answer extraction and symbolic comparison with
the reference answer. The replay cache excludes question text, raw
reasoning, and feedback text; correctness labels are not exposed to
controllers.

\subsection{Replay Evaluation and Cost Accounting}
\label{app:cost_model}

To compare policies under identical evaluation conditions, all methods
share profiles, replay pools, 1,000 evaluation seeds, and resource-accounting
rules within each benchmark split. The same seed panel is used across
discovery rounds. Seeds reorder branch consumption rather than problems;
replay is deterministic given the policy, profile, problem set, cache,
seed, and execution settings, so these seeds are not independent discovery
runs. All controller evaluations use CPU-only offline replay without
additional task-model calls.

Within each seed, accuracy and weighted compute are per-question means,
while latency is the maximum per-question replay latency defined in
Equation~\ref{eq:replay_metrics},
$\Lambda_e(\tau_u^\pi(s))=\max_{q\in\mathcal D}\ell_s(q;\pi,u)$.
This maximum is a replay statistic, not measured batch wall time.
Latency is computed from stored rollout timings: sequential phases add,
while each parallel phase contributes its cohort maximum.

We calculate weighted compute by summing think, probe, and feedback tokens
over consumed branches, with each token weighted by its model's parameter
count normalized to Qwen3-8B~\footnote{\url{https://www.alibabacloud.com/help/en/model-studio/model-pricing}}.
Profile cost budgets apply to the mean weighted compute per question.
Discovery-agent dollar expenditure is measured and reported separately
from replay inference cost.

We report joint satisfaction rate (JSR), as defined in Equation~\ref{eq:jsr}, averaged equally
over profiles within each benchmark; cross-benchmark summaries average
the two benchmark scores. Controller access restrictions are specified
in Appendix~\ref{app:action_space}.

\subsection{Requirement Profile Sampling}
\label{app:profile_sampling}

Calibration and profile generation are performed separately for AIME and
HMMT, using discovery problems only. The calibration script,
\texttt{calibrate\_profile\_grid.py}, evaluates the Cartesian product
in Table~\ref{tab:anchor_grid}, giving $6\times5\times7\times6=1{,}260$
anchor configurations per benchmark.

\begin{center}
\begin{minipage}{\linewidth}
\captionof{table}{Calibration grid. Depth is a checkpoint limit per stage;
the refinement limit counts additional stages after the initial stage.}
\label{tab:anchor_grid}
\small
\renewcommand{\arraystretch}{1.12}
\begin{tabularx}{\linewidth}{@{}lX@{}}
\toprule
Dimension & Values \\
\midrule
Model & Qwen3-0.6B, 1.7B, 4B, 8B, 14B, 32B \\
Rollout width & $\{1,4,8,16,64\}$ \\
Depth per stage & $\{1,2,4,8,16,32,66\}$ checkpoints \\
Self-refinement limit & $\{0,1,2,3,4,5\}$ \\
\bottomrule
\end{tabularx}
\end{minipage}
\end{center}

The calibration policy spawns the configured width as one parallel cohort,
uses no early stopping or pruning, and refines only branches that reach a
cached terminal stage within the depth limit. It returns a deterministic
plurality vote over branch answers. Each configuration is replayed under
100 seeds over all 60 AIME or 30 HMMT discovery problems. Within-seed
metrics are aggregated as in Equation~\ref{eq:replay_metrics}.
Let $(A_k^{\mathrm{anc}},L_k^{\mathrm{anc}},C_k^{\mathrm{anc}})$
denote the recorded calibration summary for anchor $k$.

The profile-generation script, \texttt{generate\_profiles\_stratified.py},
uses empirical quantile functions $Q_A,Q_L,Q_C$ from the anchor summaries.
Candidate thresholds are drawn independently in quantile space:
\begin{equation}
\begin{aligned}
z_A,z_L,z_C
&\stackrel{\mathrm{iid}}{\sim}\operatorname{Uniform}(0.03,0.97),\\
u
&=(Q_A(z_A),Q_L(z_L),Q_C(z_C))=(a_u,L_u,C_u).
\end{aligned}
\label{eq:impl_profile_quantiles}
\end{equation}
Its feasibility ratio is the fraction of anchor configurations whose
recorded calibration summaries satisfy all three thresholds:
\begin{equation}
\operatorname{FR}(u)=\frac{1}{1260}\sum_{k=1}^{1260}
\mathbf 1\!\left[
 A_k^{\mathrm{anc}}\ge a_u\;\land\;
 L_k^{\mathrm{anc}}\le L_u\;\land\;
 C_k^{\mathrm{anc}}\le C_u
\right].
\label{eq:impl_profile_feasibility}
\end{equation}
A lower ratio marks a harder requirement profile relative to this anchor
family. We use the ratio to stratify profile difficulty; it measures anchor
coverage rather than feasibility for every possible controller.

\begin{samepage}
Candidates are divided into five equal-frequency bins by feasibility
ratio. Within each bin, greedy maximin selection chooses the candidate
with the greatest minimum distance to the profiles already selected
in that bin, using the sampler's normalized constraint space. With
generation seed 20260901, the source set contains 100 profiles per
benchmark, with 20 selected from each bin. The 20-profile target sets
are sampled independently with distinct seeds using the same procedure;
no complete threshold triple is shared with the corresponding source set.
\par
\end{samepage}

\subsection{Discovery Settings}
\label{app:discovery_settings}

Policy discovery follows Algorithm~\ref{alg:persontts} with $B=5$
new-candidate evaluations per profile and a separate initialization
evaluation. All scored rounds are retained even after incumbent joint satisfaction rate (JSR)
reaches one. We use \texttt{claude-opus-4-6} as the discovery agent, with
a one-hour session limit and a 4,000-token thinking-budget setting.
The agent interface does not expose temperature, so the provider default
is used. Infrastructure retries have a separate allowance.
The discovery prompt is provided in Appendix~\ref{app:discovery_prompt}.
Source runs use PersonTTS (w/o reuse). On the same target
profiles, we compare PersonTTS (w/o reuse), PersonTTS
(w/o Guide), PersonTTS (w/o warm-start), and PersonTTS under the same
budget. For each PersonTTS variant, the policy with the highest JSR
on discovery problems is frozen for evaluation on held-out problems.
Outcomes on held-out problems never inform selection, while target-local
discovery feedback remains available during adaptation.

\section{Controller and Discovery Configuration}
\label{app:action_space}

Table~\ref{tab:access_roles} separates controller execution, candidate
generation, and external evaluation.
\begin{center}
\begin{minipage}{\linewidth}
\captionof{table}{Information and permissions by component.}
\label{tab:access_roles}
\small\renewcommand{\arraystretch}{1.12}
\begin{tabularx}{\linewidth}{@{}p{0.18\linewidth}X X@{}}
\toprule
Component & Inputs & Operations \\
\midrule
Controller & Requirements, model availability, remaining rollouts,
branch progress, answer histories, and resource ledger.
& Permitted replay actions only. \\
Discovery agent & Starting code, accessible code/feedback history,
bounded sanitized trace summaries, and optional Guide.
& Edits \texttt{candidate.py}; cannot execute or score it. \\
External evaluator & Submitted code, replay cache, and reference answers.
& Executes code, records costs, checks answers, and returns permitted feedback. \\
\bottomrule
\end{tabularx}
\end{minipage}
\end{center}

Static validation restricts controller code to the permitted API and
prohibits file, network, and reflection access. These restrictions do not
prevent the discovery agent from editing candidate code. Branch models
are fixed at spawning. Gold answers and held-out evaluation results are
not returned to the controller or discovery agent.

The agent interface does not expose temperature, so the provider default
is used. Infrastructure retries have a separate allowance. All conditions
retain five scored rounds, including after incumbent joint satisfaction rate (JSR) reaches one.

\section{Additional Analysis}
\label{app:Additional Analysis}

\subsection{Cross-Benchmark Generalization}
\label{sec:cross_benchmark}

To isolate how much controller structure can transfer across problem
distributions without target-specific search, we directly reuse
requirement-matched policies discovered on the other benchmark and evaluate them unchanged on the target benchmark. This setting removes the discovery agent and Guide, allowing the effect of
requirement-matched policy transfer to be examined independently.

Table~\ref{tab:cross_transfer} shows that controllers retrieved by requirement similarity remain useful even without target-benchmark discovery, indicating that the retrieved policies can retain useful controller structure under a shift in problem distribution.
Their performance nevertheless remains substantially below target-native PersonTTS, suggesting that requirement matching is best viewed as a useful initialization mechanism rather than a substitute for target-specific discovery.

\begin{table}[h]
\centering
\caption{Cross-benchmark evaluation: JSR (\%; higher is better), averaged
over 20 target profiles with 1,000 replay seeds each. Baseline mean/max
summarize the five external configurations in \Cref{tab:main_results}
plus AutoTTS-template, which is not reported there. PersonTTS reports the target-native result on the corresponding target profiles and held-out
problems in that table.}
\label{tab:cross_transfer}
\small
\setlength{\tabcolsep}{5pt}
\begin{tabular}{lcccc}
\toprule
\textbf{Target} & \begin{tabular}[c]{@{}c@{}}Baseline\\mean\end{tabular}
& \begin{tabular}[c]{@{}c@{}}Baseline\\max\end{tabular}
& \begin{tabular}[c]{@{}c@{}}Matched\\transfer\end{tabular} & PersonTTS \\
\midrule
HMMT25 ($\leftarrow$ AIME) & 4.14 & 11.51 & 25.10 & 77.28 \\
AIME26 ($\leftarrow$ HMMT) & 10.14 & 35.47 & 43.04 & 83.85 \\
\bottomrule
\end{tabular}
\end{table}

\section{What Source Comparisons Imply under New Requirements}
\label{app:reuse_analysis}
\label{app:ptts-analysis}
\label{app:ptts-rethreshold}

\subsection{Target Satisfaction Bounds}
\label{app:reuse_transfer_proof}

We prove the bound in Equation~\ref{eq:reuse_target_envelope}.
The source margin $g^\pi(s)$ uses the target thresholds, and
$r^\pi(s)$ bounds the normalized source--target execution deviation
on the same finite seed panel.

\begin{proof}
Define the minimum standardized target-execution margin
\begin{equation}
g_v^\pi(s)
=
\min_j
\frac{w_{v,j}^\pi(s)-b_j(v)}{d_j}.
\end{equation}
By assumption, each normalized target margin differs from its source
counterpart by at most $r^\pi(s)$. Their minima therefore satisfy
\begin{equation}
|g_v^\pi(s)-g^\pi(s)|
\le r^\pi(s).
\end{equation}
Hence $g^\pi(s)\ge r^\pi(s)$ implies $g_v^\pi(s)\ge0$, while
$g_v^\pi(s)\ge0$ implies $g^\pi(s)\ge-r^\pi(s)$. Therefore
\begin{equation}
\mathbf1\{g^\pi(s)\ge r^\pi(s)\}
\le
\mathbf1\{g_v^\pi(s)\ge0\}
\le
\mathbf1\{g^\pi(s)\ge-r^\pi(s)\}.
\end{equation}
Averaging over $s\in\Omega$ gives
$F_\pi^-\le f_{v,\Omega}(\pi)\le F_\pi^+$.
\end{proof}

When $r^\pi(s)=0$ for every seed, rethresholding determines the target
joint satisfaction rate (JSR) exactly. More generally, the bound depends on execution deviation
rather than profile distance alone.

\subsection{Comparing Programs on Fixed Source Records}

Fix an environment, trajectory cache, question set $\mathcal D$ with
$N\ge1$ questions, and a nonempty finite indexed replay panel $\Omega$,
with $m=|\Omega|$. Two complete programs $\pi,\pi'$ have complete,
finite outcomes $(A_u^p(s),L_u^p(s),C_u^p(s))$ under the same source
profile $u$, for $p\in\{\pi,\pi'\}$ and $s\in\Omega$, with accuracy
in $[0,1]$ and nonnegative resources.
Accuracy is the batch average, latency the maximum over questions, and
cost the mean over questions, as in Equation~\ref{eq:replay_metrics}.
Comparisons concern whole programs rather than isolated causal effects.
Seed independence is not assumed.

Let $\mathcal P=[0,1]\times[0,\infty)\times[0,\infty)$ be the space of
finite requirement thresholds. Each stored outcome defines an acceptance
region and a reclassified score:
\begin{equation}
\begin{aligned}
R_{u,s}^{p}
&=[0,A_u^p(s)]\times[L_u^p(s),\infty)\times[C_u^p(s),\infty),\\
\bar J_{u\to z}^{p}(s)&=\mathbf1\{z\in R_{u,s}^{p}\},\qquad
\bar S_{u\to z}(p)=\frac1m\sum_{s\in\Omega}\bar J_{u\to z}^{p}(s),\\
\bar\Delta_u(z)&=\bar S_{u\to z}(\pi')-\bar S_{u\to z}(\pi).
\end{aligned}
\label{eq:app-reuse-rethreshold}
\end{equation}
Bars denote threshold changes applied to \emph{fixed source outcomes},
not executions under $z$. Write
$\widehat S_z(p)=\widehat S_z(p;\mathcal D,\Omega)$ for the actual
execution joint satisfaction rate (JSR) from Equation~\ref{eq:jsr}. At the source,
$\bar S_{u\to u}(p)=\widehat S_u(p)$.
Reclassification at another profile need not equal $\widehat S_z(p)$.

\begin{pttsproposition}[Requirement-dependent comparisons of fixed records]
\label{prop:app-reuse-comparison}
Define the gained and lost success sets
$G_u(z)=\{s\in\Omega:z\in R_{u,s}^{\pi'}\setminus R_{u,s}^{\pi}\}$ and
$H_u(z)=\{s\in\Omega:z\in R_{u,s}^{\pi}\setminus R_{u,s}^{\pi'}\}$.
For every $z\in\mathcal P$,
\begin{equation}
\bar\Delta_u(z)=\frac{|G_u(z)|-|H_u(z)|}{m}.
\label{eq:app-reuse-gains}
\end{equation}
The comparison is constant within each region induced by the threshold
hyperplanes $a_z=A_u^p(s)$, $L_z=L_u^p(s)$, and $C_z=C_u^p(s)$ for
$p\in\{\pi,\pi'\}$ and $s\in\Omega$, with equality faces and domain
boundaries treated separately.
If $v$ relaxes $u$, namely $a_v\le a_u$, $L_v\ge L_u$, $C_v\ge C_u$,
let $n_p(u,v)$ count seeds unsuccessful at $u$ but successful at $v$
under reclassification. Then
\begin{equation}
\bar\Delta_u(v)=\bar\Delta_u(u)
             +\frac{n_{\pi'}(u,v)-n_{\pi}(u,v)}{m}.
\label{eq:app-reuse-relaxation}
\end{equation}
In particular, a strict source advantage $\bar\Delta_u(u)>0$ reverses
strictly if and only if
\begin{equation}
n_{\pi}(u,v)-n_{\pi'}(u,v)>m\bar\Delta_u(u).
\label{eq:app-reuse-reversal}
\end{equation}
Equality gives a tie. A smaller left-hand side preserves strict advantage.
\end{pttsproposition}

\begin{proof}
For each seed, the two joint indicators differ by $+1$ on $G_u(z)$,
$-1$ on $H_u(z)$, and zero otherwise. Summing gives
Equation~\ref{eq:app-reuse-gains}.
Within each specified region, every comparison of a threshold coordinate
with a stored outcome coordinate has fixed truth value. All joint
indicators are therefore constant there. Treating equality faces separately
handles the convention that meeting a threshold exactly counts as success.
Under relaxation, a successful fixed record cannot become unsuccessful.
Consequently,
$\bar S_{u\to v}(p)=\widehat S_u(p)+n_p(u,v)/m$.
Subtracting the identity for $\pi$ from that for $\pi'$ gives
Equation~\ref{eq:app-reuse-relaxation}. Rearranging gives the strict
reversal, tie, and preservation conditions.
\end{proof}

Each fixed-record score is nondecreasing under relaxation, while its
difference from another score need not be.
This monotonicity does not automatically extend to profile-conditioned
executions.

\subsection{Identical Marginals Can Conceal a Ranking Reversal}

Table~\ref{tab:app-jsr-coupling} gives synthetic arrays consistent with
the aggregation definitions. With $N=10$, accuracy $1$ or $0.5$ means
ten or five correct answers, and each question can have the row's latency
and cost. These are not measured records or a claim of realizability
in the particular cache.

\begin{table}[htbp]
\centering
\caption{Synthetic source records with identical marginal distributions.
Each triple is (batch accuracy, maximum question latency, mean question cost).}
\label{tab:app-jsr-coupling}
\begin{tabular}{ccc}
\toprule
Seed & $\pi$: $(A_u^\pi,L_u^\pi,C_u^\pi)$
     & $\pi'$: $(A_u^{\pi'},L_u^{\pi'},C_u^{\pi'})$\\
\midrule
1 & $(1,2,1)$   & $(1,1,1)$\\
2 & $(1,3,1)$   & $(1,3,3)$\\
3 & $(0.5,1,3)$ & $(0.5,2,1)$\\
\bottomrule
\end{tabular}
\end{table}

Choose $u=(0.75,1.5,2)$ and $v=(0.75,3.5,2)$, so only the latency ceiling
is relaxed. At $u$, $\pi$ succeeds on no seed and $\pi'$ on seed 1.
At $v$, reclassification accepts seeds 1 and 2 for $\pi$, but still only
seed 1 for $\pi'$. Source scores $(0,1/3)$ therefore become $(2/3,1/3)$,
reversing $\bar\Delta_u$ from $1/3$ to $-1/3$.
Here $n_\pi=2$ and $n_{\pi'}=0$, verifying the reversal criterion.

Yet both programs have the same marginal multisets:
$A:\{1,1,0.5\}$, $L:\{1,2,3\}$, and $C:\{1,1,3\}$.
Their cross-seed mean vectors are both $(5/6,2,5/3)$, and their marginal
pass rates at $v$ are both $(2/3,1,2/3)$.
Their identical full marginals give identical individual pass rates at
\emph{every} threshold, yet different joint scores here.
Thus even full marginals, and hence their means or individual pass rates,
cannot determine joint satisfaction rate (JSR) or comparison direction because they omit which
conditions hold together in the same seed batch.

\paragraph{A sufficient ordering condition.}
Stronger paired evidence does preserve noninferiority. If every $s$ satisfies
\begin{equation}
A_u^{\pi'}(s)\ge A_u^\pi(s),\qquad
L_u^{\pi'}(s)\le L_u^\pi(s),\qquad
C_u^{\pi'}(s)\le C_u^\pi(s),
\label{eq:app-reuse-dominance}
\end{equation}
then $\bar\Delta_u(z)\ge0$ for all $z\in\mathcal P$.
Indeed, these inequalities imply
$R_{u,s}^{\pi}\subseteq R_{u,s}^{\pi'}$, so each old success is also a
new success. Conversely, per-seed region containment implies the three
inequalities by choosing
$z=(A_u^\pi(s),L_u^\pi(s),C_u^\pi(s))$.
This equivalence is per seed. Paired dominance is only sufficient for the
aggregate ordering.
Swapping $(1,1,2)$ and $(1,2,1)$ between two seeds preserves every
aggregate JSR but violates paired dominance.
Nor does dominance imply strict improvement everywhere. At $a_z=0$ and
ceilings above all finite recorded resources, both scores equal one.

\subsection{Reclassification versus Target Execution}

Now suppose both programs are also executed under $v$ on the same
environment, cache, question set and seed panel, with complete finite
target outcomes. For $z\in\{u,v\}$, define the actual comparison and
execution discrepancies
\begin{equation}
\Delta(z)=\widehat S_z(\pi')-\widehat S_z(\pi),\qquad
\epsilon_p(u,v)=\widehat S_v(p)-\bar S_{u\to v}(p).
\end{equation}
Adding and subtracting the reclassified scores yields
\begin{equation}
\begin{aligned}
\Delta(v)-\Delta(u)
={}&\underbrace{\bar\Delta_u(v)-\bar\Delta_u(u)}_{
\text{threshold changes on source records}}\\
&+\underbrace{\epsilon_{\pi'}(u,v)-\epsilon_\pi(u,v)}_{
\text{profile-conditioned execution changes}}.
\end{aligned}
\label{eq:app-reuse-decomposition}
\end{equation}
To verify the identity, substitute
$\widehat S_v(p)=\bar S_{u\to v}(p)+\epsilon_p(u,v)$ for each program
and use $\Delta(u)=\bar\Delta_u(u)$.
The fixed-record analysis above describes only the first term.
If every paired outcome triple remains unchanged for both programs,
then each $\epsilon_p=0$. In this case, reclassification gives the exact
target joint satisfaction rate (JSR) and the fixed-record comparison applies to actual execution.
This special case is not assumed for the source bank.
The same code, cache and seed do not establish outcome invariance when
the program reads the profile.

Source comparisons support contextual revision hypotheses, while target
replay evaluates proposed programs.
This analysis does not assert complete per-seed bank records, implemented
cell enumeration or dominance checks, or sufficient Guide conditions.
It guarantees neither nearest-neighbor transfer, held-out generalization,
nor discovery-cost savings.

\subsection{What Constraint Diagnostics Identify}
\label{app:diagnostic_limits}

Fix one program, one profile, and a complete finite panel $\Omega$ of
$m\ge1$ equally weighted seeds. Let $E_A,E_L,E_C\subseteq\Omega$ contain
the seeds passing each requirement, and write
$r_i=|E_i|/m$ and
$J=|E_A\cap E_L\cap E_C|/m=\widehat S_u(\pi)$.
These are exact pass fractions, not rounded estimates.
For distinct $i,j,k\in\{A,L,C\}$, define
$H_i=|E_j\cap E_k\cap E_i^c|/m$, the fraction failing only constraint $i$.
Complements are relative to $\Omega$.

\begin{pttsproposition}[Sharp diagnostic limits]
\label{prop:diagnostic_limits}
Every feasible summary $(r_A,r_L,r_C,J)$ on an $m$-seed panel satisfies
\begin{equation}
\max\{0,r_j+r_k-1-J\}\le H_i
\le\min\{1-r_i,r_j-J,r_k-J\}.
\label{eq:diagnostic_limits}
\end{equation}
For each fixed $i$ separately, both endpoints and every intervening
multiple of $1/m$ are attained by abstract pass panels with the same
summary.
For modified sets satisfying $E_j'=E_j$ and $E_k'=E_k$, the new joint
rate obeys
$-J\le J'-J\le H_i$.
If also $E_i\subseteq E_i'$, then $0\le J'-J\le H_i$.
\end{pttsproposition}

\begin{proof}
Let $a=mr_j$, $b=mr_k$, $c=mr_i$, $t=mJ$, and
$q=|E_j\cap E_k|$. Two-set intersection bounds, containment of the
$t$ joint successes, and room for the $c-t$ other members of $E_i$ give
\[
\max\{t,a+b-m\}\le q\le\min\{a,b,t+m-c\}.
\]
Since $mH_i=q-t$, subtracting $t$ proves the bounds.
Conversely, take any integer $q$ in this interval.
Partition the panel into four sets of sizes
$q,a-q,b-q,m-a-b+q$, representing respectively the intersection of
$E_j,E_k$, their two exclusive parts, and neither.
All sizes are nonnegative.
Choose $E_i$ to contain $t$ intersection members and $c-t$ complement
members. The interval guarantees room for both choices.
This realizes the prescribed summary and $mH_i=q-t$, proving sharpness
at every allowed grid value.

Finally, preserving $E_j,E_k$ gives $0\le J'\le q/m$, hence
$-J\le J'-J\le H_i$.
If $E_i\subseteq E_i'$, no old success is lost.
Without the pass-preservation restriction, $E_i'=\emptyset$ or $\Omega$
attains the respective extreme changes at the pass-array level without
asserting executable realizability.
\end{proof}

\paragraph{When summaries suffice.}
If $r_A=1$, then $E_A=\Omega$, and hence
$H_L=r_C-J$, $H_C=r_L-J$, and $H_A=0$.
If also $r_C=1$, then $H_L=1-J=1-r_L$.
More generally, $r_C=J$ implies that any strict joint satisfaction rate (JSR) improvement must
increase the cost pass fraction, since $J'\le r_C'$.
This does not require $r_A=1$ and holds even if accuracy and latency
pass sets change.

\paragraph{When summaries remain ambiguous.}
Consider synthetic three-seed pass arrays, with coordinates $(A,L,C)$:
\[
P=((1,1,1),(0,1,1),(1,0,0)),\qquad
Q=((1,1,1),(1,0,1),(0,1,0)).
\]
Both have $r_A=r_L=r_C=2/3$ and $J=1/3$, but their $(H_A,H_L)$ values
are $(1/3,0)$ and $(0,1/3)$ respectively.
Thus even the joint score and all marginal pass rates need not identify
which isolated repair could help.
These arrays are illustrative, not measured cache outcomes.

This application of elementary finite-set intersection bounds concerns
outcome-level opportunities.
It does not implement a repair rule or assume that a continuation-cap edit
changes latency alone.
Source summaries describe source executions.
When the paired outcome triples remain unchanged, reclassification of
source records coincides with target execution.
Target replay continues to evaluate the full generated program.

\section{Prompts}\label{app:policy_prompt}\label{app:discovery_prompt}
In this section, we present the full prompt used in our policy-discovery experiments.
\begin{tcolorbox}[enhanced,breakable,
        colback=gray!5,
        colframe=black,
        boxrule=0.8pt,
        arc=2pt,
        left=10pt, right=10pt, top=10pt, bottom=10pt,
        title=\textbf{PersonTTS Policy-Discovery Agent Prompt}
    ]
        \normalfont
        \footnotesize
        \raggedright
        \setlength{\parindent}{0pt}
        System Message: \\
        You are the analysis-and-implementation Agent for one AutoTTS-style policy-discovery round. Analyze the supplied aggregate feedback and sanitized executed traces, then edit \texttt{candidate.py} in this same session. You may write only \texttt{candidate.py}. Your only permitted workspace is \texttt{\{allowed\_root\}}. For structured file tools, use absolute paths only under that exact workspace. For Bash, the working directory is already set: use relative paths only. Never use \texttt{/home/work} or inspect any parent, root, or other filesystem path.
        \par\vspace{1em}
        
        User Message: \\
        You are improving \texttt{candidate.py}, a controller for an offline mathematical-reasoning replay environment. Use one continuous Agent session to inspect the supplied search evidence and then implement the policy improvement directly in \texttt{candidate.py}.
        Work only inside the current workspace. Read \texttt{POLICY\_API.md}, \texttt{SEED\_POLICY.py}, \texttt{POLICY\_BASE.py}, \texttt{feedback.json}, \texttt{history.json}, \texttt{TRACE\_GUIDE.md}, \texttt{TRACE\_SUMMARY.json}, and \texttt{STRATEGY\_GUIDE.md}; in personalized runs also read \texttt{REQUIREMENT\_PROFILES.json}. \texttt{STRATEGY\_GUIDE.md} contains cross-profile optimization experience --- proven patterns and known failure modes from hundreds of completed runs; read it first to avoid common mistakes. \texttt{CANDIDATE\_HISTORY/} contains proposal code from completed rounds, and \texttt{TRACE\_HISTORY/} contains their complete sanitized executed traces. \texttt{SEARCH\_TRACES.jsonl} is the latest completed-round trace alias. When present, \texttt{BASELINE\_TRACES.jsonl} is the complete sanitized template-baseline trace. Do not spend tool calls searching outside these documented inputs.
        \par\vspace{0.5em}
        
        \textcolor{gray}{Hard data boundary:} \\
        - Never inspect parent directories, backing outputs, private replay caches, held-out datasets, or any filesystem path outside the supplied workspace. \\
        - Never search for question text, gold answers, raw reasoning, raw feedback, source paths, or unexecuted future checkpoints. \\
        - The trace JSONL files contain only anonymized observations and actions actually executed by evaluated policies. They remain read-only. \\
        - Do not load or `Read` a complete trace JSONL into model context. Use `TRACE\_ANALYZER.py summary` and bounded, paginated `cases --offset N --limit M` queries for statistical analysis and targeted inspection; use aggregate summaries first. \\
        - Do not add file, network, subprocess, reflection, dynamic import, or external private-attribute access to the candidate. Deterministic helpers defined on `CandidatePolicy` are allowed.
        Trace inspection is evidence gathering, not a coverage proof. Inspect the aggregate results, then autonomously choose useful anomalous groups and representative traces. You are not required to claim that every record entered context, report `records\_analyzed`, or create an analysis report. After sufficient analysis, edit `candidate.py` in this same session. Finish only after it contains one valid `CandidatePolicy`; do not merely describe a plan.
        \par\vspace{0.5em}
        
        \textcolor{gray}{Model availability rule:} \\
        - `POLICY\_API.md` and runtime `env.models` are authoritative for the complete model pool available to a policy. \\
        - Models observed in traces only show what the incumbent or historical candidate actually spawned. Never infer that an observed subset, including only Qwen3-4B, is the available model set.
        \textcolor{gray}{Algorithm contract:} \\
        - A branch chooses its model in `spawn(model)` and can never switch models. \\
        - `refine(branch\_id)` is legal only when the current stage is terminal. \\
        - `env.requirements` is the current user's explicit accuracy, latency, and weighted-compute hard-requirement profile. Use all three thresholds directly to choose model, width, checkpoints, and refine behavior. \\
        - Personalized discovery has no `beta`. Do not invent or infer a replacement scalar from the three explicit requirements. \\
        - Public observations and action return values are dictionaries. Use bracket access such as `state["branches"]` and `branch["branch\_id"]`. \\
        - Maximize the explicit utility in `feedback.json`: the profile-weighted fraction of complete seed batches that simultaneously satisfy all three hard constraints. A proposal is accepted only when this utility strictly exceeds the incumbent utility. \\
        - Latency follows replay cohort semantics: sequential actions add; each `*\_many` phase adds the maximum cohort latency. Weighted compute includes think, probe, and feedback tokens weighted by model size. \\
        - Base decisions only on public observations already revealed by executed actions, and keep the controller deterministic for a fixed replay seed.
        Only `candidate.py` is writable. Do not edit the API, template/baseline copies, profiles, feedback, history, candidate history, analyzer, summaries, or traces. Do not claim performance absent from the supplied evidence.
    \end{tcolorbox}

    {
\captionsetup{type=figure}
\captionof{figure}{\textbf{Policy-discovery prompt.} The System Message and User Message are presented together in the single framed prompt box used for one policy-improvement round.}
    \label{fig:policy_discovery_prompt}
}

\section{Public Policy API}\label{app:public_api}
\begin{tcolorbox}[enhanced,breakable,
        colback=gray!5,
        colframe=black,
        boxrule=0.8pt,
        arc=2pt,
        left=10pt, right=10pt, top=10pt, bottom=10pt,
        title=\textbf{Public Policy API}
    ]
        \normalfont
        \footnotesize
        \raggedright
        \setlength{\parindent}{0pt}
        CandidatePolicy(config=\{"requirements": profile\}) must expose \texttt{solve(self, env)} and return an observed answer or \texttt{None}. Personalized runs expose only explicit requirements and resources; legacy beta-grid runs separately receive \texttt{beta}.
        \par\vspace{0.8em}
        
        The policy may use:
        - \texttt{env.models}: the authoritative complete set of model names available in the current episode. Models observed in baseline or historical traces only describe what those policies spawned and must never be used to narrow this set. \\
        - \texttt{env.beta}: available only in legacy beta-grid evaluation. Personalized discovery has no beta field in the environment, observation, or policy config. \\
        - \texttt{env.requirements}: an isolated copy of the current profile, containing exactly one hard constraint for accuracy, latency, and weighted compute. The canonical \texttt{constraints} list is preserved, and each constraint is also available directly by metric name (for example, \texttt{env.requirements["latency\_s"]["value"]}). It is \texttt{None} only in legacy beta evaluation. \\
        - \texttt{env.observe()}: current public state with model pool counts, opened branches, and charged cost. \\
        - \texttt{env.spawn(model)}: create one model-bound branch and reveal checkpoint 0. \\
        - \texttt{env.spawn\_many(model, count)}: atomically open a fixed-model cohort. \\
        - \texttt{env.continue\_branch(branch\_id)}: reveal the next checkpoint in the current stage. \\
        - \texttt{env.continue\_many(branch\_ids)}: atomically advance one synchronous checkpoint round. \\
        - \texttt{env.refine(branch\_id)}: after a terminal stage, charge feedback tokens and enter the next stage without revealing its first checkpoint. \\
        - \texttt{env.refine\_many(branch\_ids)}: atomically refine a survivor cohort after every member is terminal. \\
        - \texttt{env.can\_afford\_complete\_candidates(...)}: report only whether a complete configured candidate/window fits the remaining pool and weighted-compute budget. \\
        - \texttt{env.prune(branch\_id)}: deactivate a branch. \\
        - \texttt{env.finish(answer=None)}: terminate using an observed answer, or use deterministic plurality when omitted.
        \par\vspace{0.6em}
        
        \texttt{env.observe()} and every state-changing action return Python dictionaries. Use bracket access for returned fields.
        \par\vspace{0.4em}
        
        \texttt{state = env.observe()} \\
        \texttt{for branch in state["branches"]:} \\
        \texttt{\hspace*{2.0em}if branch["can\_continue"]:} \\
        \texttt{\hspace*{4.0em}env.continue\_branch(branch["branch\_id"])}
        \par\vspace{0.4em}
        
        \texttt{opened = env.spawn\_many(env.models[0], 4)} \\
        \texttt{env.continue\_many(} \\
        \texttt{\hspace*{2.0em}[branch["branch\_id"] for branch in opened if branch["can\_continue"]]} \\
        \texttt{)}
        \par\vspace{0.6em}
        
        Each public branch dictionary contains \texttt{branch\_id}, \texttt{model}, \texttt{stage\_index}, \texttt{stage\_kind}, \texttt{checkpoint\_index}, \texttt{latest\_answer}, \texttt{answer\_history}, \texttt{stage\_terminal}, \texttt{can\_continue}, \texttt{can\_refine}, and \texttt{pruned}. No question, gold answer, raw generation, raw feedback, source path, future answer, or future token delta is public.
        In personalized discovery, \texttt{REQUIREMENT\_PROFILES.json} is the immutable profile population. Each seed evaluates one complete configured discovery-problem batch for every profile. A batch scores 1 only when all three constraints are satisfied; the discovery utility is the profile-weighted batch satisfaction rate. There is no compatibility beta in this mode: use the three thresholds directly.
        \texttt{SEED\_POLICY.py} contains the adapted AutoTTS baselines \texttt{ASCConfigured4D}, \texttt{ESCConfigured4D}, and \texttt{ParallelProbeSR}. Their standard evaluation uses an explicit fixed \texttt{ResourceProfile}; it does not map \texttt{beta} to a hand-authored model/width/depth/stage schedule.
        Allowed imports are \texttt{collections}, \texttt{math}, \texttt{statistics}, \texttt{typing}, \texttt{policies.base}, and \texttt{policies.seeds}. Candidate validation rejects I/O, subprocesses, networking, dynamic imports, reflection, external private attributes, and forbidden data references. A constructor may call \texttt{super().\_\_init\_\_(...)}; deterministic private helpers defined on \texttt{CandidatePolicy} are also allowed.
    \end{tcolorbox}

    {
\captionsetup{type=figure}
\captionof{figure}{\textbf{Public policy API.} The public environment contract, state fields, actions, personalized utility definition, and candidate restrictions used by the offline replay controller.}
    \label{fig:public_policy_api}
}

\section{Meta-Prompt for Strategy Guide Generation}\label{app:guide_generation}
\begin{tcolorbox}[enhanced,breakable,
        colback=gray!5,
        colframe=black,
        boxrule=0.8pt,
        arc=2pt,
        left=8pt, right=8pt, top=6pt, bottom=6pt,
        title=\textbf{Strategy Guide Generation from the Discovery Experience Bank}
    ]
        \normalfont
        \footnotesize
        \raggedright
        \setlength{\parindent}{0pt}
        You are a research agent tasked with distilling cross-user procedural knowledge from a historical policy-discovery experience bank. Produce a compact natural-language skill that helps a discovery agent design and revise personalized inference policies for new user profiles.
        \par\vspace{0.35em}
        
        \textcolor{gray}{Objective and scope:} \\
        Distill how to search, rather than which historical policy to copy. Improve initial candidate design, diagnosis of evaluation feedback, and subsequent revisions under user-specific accuracy, latency, and compute requirements. Use the original evaluator's joint profile satisfaction rate as the policy-quality objective. Preserve metric directions, aggregation levels, evaluation units, and profile weights. Per-constraint satisfaction rates and margins are diagnostic information, not replacement objectives. Distinguish policy satisfaction from proposal acceptance, preserve high-quality rejected candidates, and do not classify ceiling-score ties as policy failures.
        Do not modify the evaluator, policy interface, model weights, candidate-selection rule, or downstream discovery protocol. Do not introduce historical-policy retrieval or warm-start initialization. Under a fixed-round protocol, preserve all rounds, including score plateaus. Use only authorized source experience and supplied system specifications; keep original artifacts read-only and do not launch new discovery experiments or unconfigured external model calls.
        \par\vspace{0.6em}
        
        \textcolor{gray}{Inventory, alignment, and scale-adaptive processing:} \\
        Inventory authorized profiles, runs, rounds, environments, and evaluation configurations. Recover candidate, incumbent, parent, and feedback relationships from artifact identities, code versions or hashes, logs, and evaluation configurations; equal utility values alone do not verify identity. Do not assume a fixed bank size, balanced strata, complete trajectories, or that feedback stored in a round evaluates that round's proposal.
        Build a compact case index containing profile requirements, pre-generation feedback, candidate identity, design summary, joint satisfaction, constraint diagnostics, execution signals, and source pointers. When the bank exceeds the context budget, process histories in a reproducible configured order, preserve temporal context, persist intermediate summaries with source pointers, and merge rules across batches. Recompute statistics from unique underlying records, retain legitimate repeated runs, revisit source cases supporting retained rules, and report exact coverage, omissions, exclusions, ordering, batching, and any configured subsampling.
        \par\vspace{0.6em}
        
        \textcolor{gray}{Analyze discovery experience:} \\
        Include initial designs, improvements, unchanged outcomes, regressions, and recoveries. Prioritize within-profile comparisons under compatible evaluation conditions. Distinguish satisfaction changes from progress relative to the incumbent, track numerical versus structural changes, and account descriptively for requirements, environments, discovery stage, and incumbent level. Aggregate within profiles before averaging across profiles, preserve dependence across repeated rounds, and do not reconstruct joint satisfaction from marginal rates.
        Inspect actual code and available execution traces for evidence-bearing cases. Static parsing can identify patterns but does not establish runtime behavior; a named variable, API call, or confidence condition is not proof of execution. Describe simultaneous changes as bundles unless attribution is supported. Treat undetected behavior as unknown, not absent.
        \par\vspace{0.6em}
        
        \textcolor{gray}{Abstract and audit conditional rules:} \\
        For every retained rule, establish the chain: observable requirements and prior feedback, diagnostic hypothesis, proposed design or revision, execution and joint-satisfaction outcomes, applicability conditions, and limitations. Search for supporting, ineffective, and contradictory cases. Separate measured outcomes from plausible explanations and label evidence as MECHANISM, OBSERVATION, HYPOTHESIS, or INCONCLUSIVE.
        Express each rule with a Rule ID and title, Trigger, Evidence, Action, and Verify. The trigger must be observable during discovery; the evidence must include supporting statistics or mechanisms and source pointers; the action must provide a concrete design choice and fallback ordering; and verification must compare the resulting joint satisfaction with the incumbent. Rules must guide both initial design and later feedback-driven revisions without requiring hidden labels, inaccessible feasibility information, or fixed benchmark identities.
        \par\vspace{0.6em}
        
        \textcolor{gray}{Separate contracts, calibration, and experience:} \\
        Organize the generated skill into Environment Contracts, Local Calibration, and Experience-Derived Guidance. Keep interface invariants distinct from empirical preferences. Check units and aggregation levels in every formula; distinguish means from tail quantiles, action costs from episode costs, token counts from checkpoint counts, and loop limits from realized depth. Treat approximate estimates as lower bounds or heuristics rather than feasibility certificates.
        Allow joint consideration of model, width, depth, and stopping behavior. Handle legal shallow configurations, missing estimates, empty answers, and state refreshes according to the actual interface. Evaluate revisions by joint satisfaction: a change in one constraint is not sufficient by itself to reject a revision that improves the joint objective.
        \par\vspace{0.6em}
        
        \textcolor{gray}{Deliverables and final review:} \\
        Produce a self-contained English \texttt{STRATEGY\_GUIDE.md} within the configured length limit, organized as Environment Contracts, Local Calibration, Conditional Search Heuristics, Decision Flow, and Failure Diagnosis. Produce \texttt{SKILL\_EVIDENCE.md} using the same rule IDs, with source coverage, definitions, supporting and contradictory cases, artifact locations, limitations, and processing provenance. Do not reproduce the entire bank, include historical answer lookups, force one controller template, or claim improved efficiency or generalization before independent evaluation.
        After drafting, perform one consolidated review of evidence provenance, arithmetic, units, boundary cases, and consistency across rules, examples, decision flow, and diagnosis tables. Correct unsupported claims without launching effectiveness experiments. Save the reviewed outputs, freeze them, and conclude with a brief coverage and limitation statement.
    \end{tcolorbox}

    {
\captionsetup{type=figure}
\captionof{figure}{\textbf{Strategy-guide generation meta-prompt.} A provenance-aware procedure for distilling conditional cross-user controller-design knowledge from authorized discovery experience while separating environment contracts, local calibration, and empirical guidance.}
    \label{fig:strategy_guide_meta_prompt}
}

\section{General Methodology from Training-Set Discovery Traces}
\label{App:training_trace_methodology}
We distill the training-set discovery traces into a controller-design
workflow. The summary separates replay contracts, local calibration,
search heuristics, and failure diagnosis so that the design can be applied
to new user profiles with profile-specific re-estimation.
\begin{tcolorbox}[enhanced,breakable,
        colback=gray!5,
        colframe=black,
        boxrule=0.8pt,
        arc=2pt,
        left=10pt, right=10pt, top=10pt, bottom=10pt,
        title=\textbf{Controller Design Methodology from Training-Set Traces}
    ]
        \normalfont
        \footnotesize
        \raggedright
        \setlength{\parindent}{0pt}
        \textcolor{gray}{1. Evidence and scope:} \\
        Read the strategy guide before inspecting traces. Use aggregate summaries first, then inspect bounded representative cases and anomalous groups. The evidence consists of executed training-set policies; raw questions, gold answers, private caches, and unexecuted checkpoints remain outside the evidence.
        \par\vspace{0.6em}
        
        \textcolor{gray}{2. Calibrate the replay cost model:} \\
        Estimate the mean phase latency from the trace summary and estimate tokens per continue action from aggregate action counts. Use the authoritative model weights and the current \texttt{env.requirements}. Parallel \texttt{spawn\_many} and \texttt{continue\_many} use cohort maxima for latency, whereas weighted compute is summed over all branches. These estimates are heuristic because the scorer uses maximum per-question latency within each seed batch and accounts for action-level variation.
        \par\vspace{0.6em}
        
        \textcolor{gray}{3. Optimize the joint hard-constraint utility:} \\
        A seed batch is satisfied only when accuracy, latency, and weighted compute all pass simultaneously. The optimization target is the profile-weighted fraction of satisfied batches, and proposal acceptance is strict improvement over the incumbent. Per-constraint satisfaction rates and margins diagnose failures; joint satisfaction rate (JSR) remains the acceptance criterion.
        \par\vspace{0.6em}
        
        \textcolor{gray}{4. Construct a feasible policy before optimizing quality:} \\
        Try several \texttt{(width, loop\_limit)} combinations at the current model using the calibrated estimates. If none is feasible, move to a lighter model. The default controller structure is parallel initial spawning, parallel continuation rounds, and deterministic plurality voting over usable final answers. A zero continuation limit is valid when the budget permits only spawn actions.
        \par\vspace{0.6em}
        
        \textcolor{gray}{5. Match depth and width to the bottleneck:} \\
        In the observed math-reasoning traces, accuracy is the most common bottleneck. Increase depth when branches remain active at the loop limit and accuracy is binding; reduce depth when resources bind and branches terminate earlier. Width primarily adds voting diversity and should increase only when compute headroom supports it.
        \par\vspace{0.6em}
        
        \textcolor{gray}{6. Treat early stopping and refine as conditional mechanisms:} \\
        Use confidence-based early stopping or refinement only when the relevant budget and failure diagnosis justify it. Early stopping trades latency for accuracy, while refinement adds tokens and latency; verify the joint utility after either change.
        \par\vspace{0.6em}
        
        \textcolor{gray}{7. Diagnose failures before changing parameters:} \\
        After a failed round, read per-constraint satisfaction rates and margins, then check whether the loop limit was actually reached. Small negative margins call for an incremental change of one or two rounds or branches; large gaps call for recalibration from the new trace summary. Do not assume that fixing the lowest individual satisfaction rate will improve the joint rate, because different seed batches can fail on different constraints.
        \par\vspace{0.6em}
        
        \textcolor{gray}{8. Reallocate only verified headroom:} \\
        Latency headroom can support a small increase in continuation depth, and compute headroom can support a small increase in branch width, but both changes must be checked against the other constraints. Reducing depth or width to exploit accuracy headroom is risky because terminal-answer quality can change non-linearly.
        \par\vspace{0.6em}
        
        \textcolor{gray}{9. Generalization and limits:} \\
        The guide summarizes 488 failed proposals, 800 observed transitions, and 200 training profiles. These statistics define its empirical scope; preserve determinism, use public observations, and re-estimate costs and joint satisfaction after each meaningful policy change.
    \end{tcolorbox}

    {
\captionsetup{type=figure}
\captionof{figure}{\textbf{General controller-design methodology.} A training-trace-derived workflow for calibrating replay costs, satisfying explicit accuracy--latency--compute constraints, diagnosing failures, and making conservative policy updates. Empirical statistics are confined to the training-set discovery evidence and should not be interpreted as universal laws.}
    \label{fig:training_trace_methodology}
}

%% file: iclr2027_conference.bib
@misc{zheng2026dreamrsirecursiveselfimprovementevolving,
      title={Dream-RSI: Recursive Self-Improvement through Evolving Worlds}, 
      author={Tong Zheng and Xidong Wu and Zheng Zhang and Zhankui He and Chaoyi Zhang and Benjamin Coleman and Ruoqiao Wei and Di Bai and Haolin Liu and Rui Liu and Xue Wang and Yue Zhuan and Wang-Cheng Kang and Renkai Xiang and Heng Huang and Xinwu Cheng and Yunsong Guo},
      year={2026},
      eprint={2609.14858},
      archivePrefix={arXiv},
      primaryClass={cs.CL},
      url={https://arxiv.org/abs/2609.14858}, 
}

@inproceedings{ASC,
  title={Let’s Sample Step by Step: Adaptive-Consistency for Efficient Reasoning and Coding with LLMs},
  author={Aggarwal, Pranjal and Madaan, Aman and Yang, Yiming and others},
  booktitle={Proceedings of the 2023 Conference on Empirical Methods in Natural Language Processing},
  pages={12375--12396},
  year={2023}
}

@inproceedings{ESC,
  title={Escape sky-high cost: Early-stopping self-consistency for multi-step reasoning},
  author={Li, Yiwei and Yuan, Peiwen and Feng, Shaoxiong and Pan, Boyuan and Wang, Xinglin and Sun, Bin and Wang, Heda and Li, Kan},
  booktitle={International Conference on Learning Representations},
  volume={2024},
  pages={14751--14768},
  year={2024}
}

@inproceedings{SC,
  author       = {Xuezhi Wang and
                  Jason Wei and
                  Dale Schuurmans and
                  Quoc V. Le and
                  Ed H. Chi and
                  Sharan Narang and
                  Aakanksha Chowdhery and
                  Denny Zhou},
  title        = {Self-Consistency Improves Chain of Thought Reasoning in Language Models},
  booktitle    = {The Eleventh International Conference on Learning Representations,
                  {ICLR} 2023, Kigali, Rwanda, May 1-5, 2023},
  publisher    = {OpenReview.net},
  year         = {2023},
  url          = {https://openreview.net/pdf?id=1PL1NIMMrw},
  bibsource    = {dblp computer science bibliography, https://dblp.org}
}

@inproceedings{MATH,
  author       = {Dan Hendrycks and
                  Collin Burns and
                  Saurav Kadavath and
                  Akul Arora and
                  Steven Basart and
                  Eric Tang and
                  Dawn Song and
                  Jacob Steinhardt},
  title        = {Measuring Mathematical Problem Solving With the {MATH} Dataset},
  booktitle    = {Proceedings of the Neural Information Processing Systems Track on
                  Datasets and Benchmarks 1, NeurIPS Datasets and Benchmarks 2021, December
                  2021, virtual},
  year         = {2021},
  url          = {https://datasets-benchmarks-proceedings.neurips.cc/paper/2021/hash/be83ab3ecd0db773eb2dc1b0a17836a1-Abstract-round2.html},
  bibsource    = {dblp computer science bibliography, https://dblp.org}
}

@inproceedings{wu2025inference,
  title={Inference scaling laws: An empirical analysis of compute-optimal inference for LLM problem-solving},
  author={Wu, Yangzhen and Sun, Zhiqing and Li, Shanda and Welleck, Sean and Yang, Yiming},
  booktitle={The Thirteenth International Conference on Learning Representations},
  year={2025}
}

@article{brown2024large,
  title={Large language monkeys: Scaling inference compute with repeated sampling},
  author={Brown, Bradley and Juravsky, Jordan and Ehrlich, Ryan and Clark, Ronald and Le, Quoc V and R{\'e}, Christopher and Mirhoseini, Azalia},
  journal={arXiv preprint arXiv:2407.21787},
  year={2024}
}

@inproceedings{lightman2024let,
  title={Let's verify step by step},
  author={Lightman, Hunter and Kosaraju, Vineet and Burda, Yuri and Edwards, Harrison and Baker, Bowen and Lee, Teddy and Leike, Jan and Schulman, John and Sutskever, Ilya and Cobbe, Karl},
  booktitle={International Conference on Learning Representations},
  volume={2024},
  pages={39578--39601},
  year={2024}
}

@misc{o1,
  title  = {Learning to Reason with LLMs},
  author = {OpenAI},
  url    = {https://openai.com/index/learning-to-reason-with-llms/},
  year   = {2024}
}

@article{ToT,
  title={Tree of thoughts: Deliberate problem solving with large language models},
  author={Yao, Shunyu and Yu, Dian and Zhao, Jeffrey and Shafran, Izhak and Griffiths, Tom and Cao, Yuan and Narasimhan, Karthik},
  journal={Advances in neural information processing systems},
  volume={36},
  pages={11809--11822},
  year={2023}
}

@article{Kimi-k1.5,
  title   = {Kimi k1.5: Scaling Reinforcement Learning with LLMs},
  author  = {{Kimi Team} and Du, Angang and Gao, Bofei and Xing, Bowei and Jiang, Changjiu and Chen, Cheng and Li, Cheng and Xiao, Chenjun and Du, Chenzhuang and Liao, Chonghua and others},
  journal = {arXiv preprint arXiv:2501.12599},
  year    = {2025}
}

@article{DeepSeek-R1,
  title   = {DeepSeek-R1: Incentivizing Reasoning Capability in LLMs via Reinforcement Learning},
  author  = {{DeepSeek-AI} and Daya Guo and Dejian Yang and Haowei Zhang and Junxiao Song and Ruoyu Zhang and Runxin Xu and Qihao Zhu and Shirong Ma and Peiyi Wang and Xiao Bi and Xiaokang Zhang and Xingkai Yu and Yu Wu and Z. F. Wu and Zhibin Gou and Zhihong Shao and Zhuoshu Li and Ziyi Gao and Aixin Liu and Bing Xue and Bingxuan Wang and Bochao Wu and Bei Feng and Chengda Lu and Chenggang Zhao and Chengqi Deng and Chenyu Zhang and Chong Ruan and Damai Dai and Deli Chen and Dongjie Ji and Erhang Li and Fangyun Lin and Fucong Dai and Fuli Luo and Guangbo Hao and Guanting Chen and Guowei Li and H. Zhang and Han Bao and Hanwei Xu and Haocheng Wang and Honghui Ding and Huajian Xin and Huazuo Gao and Hui Qu and Hui Li and Jianzhong Guo and Jiashi Li and Jiawei Wang and Jingchang Chen and Jingyang Yuan and Junjie Qiu and Junlong Li and J. L. Cai and Jiaqi Ni and Jian Liang and Jin Chen and Kai Dong and Kai Hu and Kaige Gao and Kang Guan and Kexin Huang and Kuai Yu and Lean Wang and Lecong Zhang and Liang Zhao and Litong Wang and Liyue Zhang and Lei Xu and Leyi Xia and Mingchuan Zhang and Minghua Zhang and Minghui Tang and Meng Li and Miaojun Wang and Mingming Li and Ning Tian and Panpan Huang and Peng Zhang and Qiancheng Wang and Qinyu Chen and Qiushi Du and Ruiqi Ge and Ruisong Zhang and Ruizhe Pan and Runji Wang and R. J. Chen and R. L. Jin and Ruyi Chen and Shanghao Lu and Shangyan Zhou and Shanhuang Chen and Shengfeng Ye and Shiyu Wang and Shuiping Yu and Shunfeng Zhou and Shuting Pan and S. S. Li and Shuang Zhou and Shaoqing Wu and Shengfeng Ye and Tao Yun and Tian Pei and Tianyu Sun and T. Wang and Wangding Zeng and Wanjia Zhao and Wen Liu and Wenfeng Liang and Wenjun Gao and Wenqin Yu and Wentao Zhang and W. L. Xiao and Wei An and Xiaodong Liu and Xiaohan Wang and Xiaokang Chen and Xiaotao Nie and Xin Cheng and Xin Liu and Xin Xie and Xingchao Liu and Xinyu Yang and Xinyuan Li and Xuecheng Su and Xuheng Lin and X. Q. Li and Xiangyue Jin and Xiaojin Shen and Xiaosha Chen and Xiaowen Sun and Xiaoxiang Wang and Xinnan Song and Xinyi Zhou and Xianzu Wang and Xinxia Shan and Y. K. Li and Y. Q. Wang and Y. X. Wei and Yang Zhang and Yanhong Xu and Yao Li and Yao Zhao and Yaofeng Sun and Yaohui Wang and Yi Yu and Yichao Zhang and Yifan Shi and Yiliang Xiong and Ying He and Yishi Piao and Yisong Wang and Yixuan Tan and Yiyang Ma and Yiyuan Liu and Yongqiang Guo and Yuan Ou and Yuduan Wang and Yue Gong and Yuheng Zou and Yujia He and Yunfan Xiong and Yuxiang Luo and Yuxiang You and Yuxuan Liu and Yuyang Zhou and Y. X. Zhu and Yanhong Xu and Yanping Huang and Yaohui Li and Yi Zheng and Yuchen Zhu and Yunxian Ma and Ying Tang and Yukun Zha and Yuting Yan and Z. Z. Ren and Zehui Ren and Zhangli Sha and Zhe Fu and Zhean Xu and Zhenda Xie and Zhengyan Zhang and Zhewen Hao and Zhicheng Ma and Zhigang Yan and Zhiyu Wu and Zihui Gu and Zijia Zhu and Zijun Liu and Zilin Li and Ziwei Xie and Ziyang Song and Zizheng Pan and Zhen Huang and Zhipeng Xu and Zhongyu Zhang and Zhen Zhang},
  journal = {arXiv preprint arXiv:2501.12948},
  year    = {2025}
}

@article{wang2026time,
  title={On Time, Within Budget: Constraint-Driven Online Resource Allocation for Agentic Workflows},
  author={Wang, Xinglin and Liu, Zishen and Feng, Shaoxiong and Yuan, Peiwen and Li, Yiwei and Shi, Jiayi and Zhang, Yueqi and Tan, Chuyi and Zhang, Ji and Pan, Boyuan and others},
  journal={arXiv preprint arXiv:2605.06110},
  year={2026}
}

@article{huang2025latency,
  title={Latency and token-aware test-time compute},
  author={Huang, Jenny Y and Damani, Mehul and El-Kurdi, Yousef and Astudillo, Ramon and Sun, Wei},
  journal={arXiv preprint arXiv:2509.09864},
  year={2025}
}

@inproceedings{Self-Refine,
  title     = {Self-Refine: Iterative Refinement with Self-Feedback},
  author    = {Madaan, Aman and Tandon, Niket and Gupta, Prakhar and Hallinan, Skyler and Gao, Luyu and Wiegreffe, Sarah and Alon, Uri and Dziri, Nouha and Prabhumoye, Shrimai and Yang, Yiming and Gupta, Shashank and Majumder, Bodhisattwa Prasad and Hermann, Katherine and Welleck, Sean and Yazdanbakhsh, Amir and Clark, Peter},
  booktitle = {Advances in Neural Information Processing Systems (NeurIPS)},
  pages     = {46534--46594},
  volume    = {36},
  year      = {2023}
}

@inproceedings{wang2025make,
  title={Make every penny count: Difficulty-adaptive self-consistency for cost-efficient reasoning},
  author={Wang, Xinglin and Feng, Shaoxiong and Li, Yiwei and Yuan, Peiwen and Zhang, Yueqi and Tan, Chuyi and Pan, Boyuan and Hu, Yao and Li, Kan},
  booktitle={Findings of the Association for Computational Linguistics: NAACL 2025},
  pages={6919--6932},
  year={2025}
}

@inproceedings{zhang2025aflow,
  title={Aflow: Automating agentic workflow generation},
  author={Zhang, Jiayi and Xiang, Jinyu and Yu, Zhaoyang and Teng, Fengwei and Chen, Xionghui and Chen, Jiaqi and Zhuge, Mingchen and Cheng, Xin and Hong, Sirui and Wang, Jinlin and others},
  booktitle={International Conference on Learning Representations},
  volume={2025},
  pages={34040--34077},
  year={2025}
}

@article{lee2026meta,
  title={Meta-harness: End-to-end optimization of model harnesses},
  author={Lee, Yoonho and Nair, Roshen and Zhang, Qizheng and Lee, Kangwook and Khattab, Omar and Finn, Chelsea},
  journal={arXiv preprint arXiv:2603.28052},
  year={2026}
}

@article{dekoninck2026matharena,
      title={Beyond Benchmarks: MathArena as an Evaluation Platform for Mathematics with LLMs}, 
      author={Jasper Dekoninck and Nikola Jovanović and Tim Gehrunger and Kári Rögnvaldsson and Ivo Petrov and Chenhao Sun and Martin Vechev},
      year={2026},
      eprint={2605.00674},
      archivePrefix={arXiv},
      primaryClass={cs.CL},
      url={https://arxiv.org/abs/2605.00674}, 
}

@article{wang2025faster,
  title={Faster and better llms via latency-aware test-time scaling},
  author={Wang, Zili and Zhang, Tianyu and Bai, Haoli and Hou, Lu and Yu, Xianzhi and Liu, Wulong and Xiang, Shiming and Zhu, Lei},
  journal={arXiv preprint arXiv:2505.19634},
  year={2025}
}

@article{qwen3technicalreport,
  title={Qwen3 technical report},
  author={Yang, An and Li, Anfeng and Yang, Baosong and Zhang, Beichen and Hui, Binyuan and Zheng, Bo and Yu, Bowen and Gao, Chang and Huang, Chengen and Lv, Chenxu and others},
  journal={arXiv preprint arXiv:2505.09388},
  year={2025}
}

@article{wang2026share,
  title={Share More, Search Less: Collaborative Parallel Thinking for Efficient Test-Time Scaling},
  author={Wang, Xinglin and Lin, Hao and Feng, Shaoxiong and Yuan, Peiwen and Li, Yiwei and Shi, Jiayi and Zhang, Yueqi and Tan, Chuyi and Zhang, Ji and Pan, Boyuan and others},
  journal={arXiv preprint arXiv:2605.27030},
  year={2026}
}

@inproceedings{snell2025scaling,
  title={Scaling LLM test-time compute optimally can be more effective than scaling parameters for reasoning},
  author={Snell, Charlie and Lee, Jaehoon and Xu, Kelvin and Kumar, Aviral},
  booktitle={International Conference on Learning Representations},
  volume={2025},
  pages={10131--10165},
  year={2025}
}

@article{zheng2026parallelprobe,
  title={Parallel-probe: Towards efficient parallel thinking via 2d probing},
  author={Zheng, Tong and Huang, Chengsong and Dai, Runpeng and He, Yun and Liu, Rui and Ni, Xin and Bao, Huiwen and Wang, Kaishen and Zhu, Hongtu and Huang, Jiaxin and others},
  journal={arXiv preprint arXiv:2602.03845},
  year={2026}
}

@article{zheng2026autotts,
  title={LLMs Improving LLMs: Agentic Discovery for Test-Time Scaling},
  author={Zheng, Tong and Liu, Haolin and Huang, Chengsong and Bao, Huiwen and Zhang, Sheng and Liu, Rui and Dai, Runpeng and Chen, Ruibo and Liu, Chenxi and Xiong, Tianyi and others},
  journal={arXiv preprint arXiv:2605.08083},
  year={2026}
}

@article{du2026swift,
  title={Why Search When You Can Transfer? Amortized Agentic Workflow Design from Structural Priors},
  author={Du, Shiyi and Liu, Jiayuan and Du, Weihua and Huang, Yue and Li, Jiayi and Luo, Yingtao and Zhang, Xiangliang and Conitzer, Vincent and Kingsford, Carl},
  journal={arXiv preprint arXiv:2604.25012},
  year={2026}
}

@article{yuan2026flowbank,
  title={FlowBank: Query-Adaptive Agentic Workflows Optimization through Precompute-and-Reuse},
  author={Yuan, Lingzhi and Deng, Chenghao and Yu, Fangxu and Chakraborty, Souradip and Rostami, Mohammad and Huang, Furong},
  journal={arXiv preprint arXiv:2606.11290},
  year={2026}
}

@article{wang2026every,
  title={Every rollout counts: Optimal resource allocation for efficient test-time scaling},
  author={Wang, Xinglin and Li, Yiwei and Feng, Shaoxiong and Yuan, Peiwen and Zhang, Yueqi and Shi, Jiayi and Tan, Chuyi and Pan, Boyuan and Hu, Yao},
  journal={Advances in Neural Information Processing Systems},
  volume={38},
  pages={102312--102338},
  year={2026}
}

@article{wang2026rse,
  title={Do not waste your rollouts: Recycling search experience for efficient test-time scaling},
  author={Wang, Xinglin and Shi, Jiayi and Feng, Shaoxiong and Yuan, Peiwen and Li, Yiwei and Zhang, Yueqi and Tan, Chuyi and Zhang, Ji and Pan, Boyuan and Hu, Yao and others},
  journal={arXiv preprint arXiv:2601.21684},
  year={2026}
}

@article{funsearch,
  title={Mathematical discoveries from program search with large language models},
  author={Romera-Paredes, Bernardino and Barekatain, Mohammadamin and Novikov, Alexander and Balog, Matej and Kumar, M Pawan and Dupont, Emilien and Ruiz, Francisco JR and Ellenberg, Jordan S and Wang, Pengming and Fawzi, Omar and others},
  journal={Nature},
  volume={625},
  number={7995},
  pages={468--475},
  year={2024},
  publisher={Nature Publishing Group UK London}
}

@article{eoh,
  title={Evolution of heuristics: Towards efficient automatic algorithm design using large language model},
  author={Liu, Fei and Tong, Xialiang and Yuan, Mingxuan and Lin, Xi and Luo, Fu and Wang, Zhenkun and Lu, Zhichao and Zhang, Qingfu},
  journal={arXiv preprint arXiv:2401.02051},
  year={2024}
}

@article{novikov2025alphaevolve,
  title={Alphaevolve: A coding agent for scientific and algorithmic discovery},
  author={Novikov, Alexander and V{\~u}, Ng{\^a}n and Eisenberger, Marvin and Dupont, Emilien and Huang, Po-Sen and Wagner, Adam Zsolt and Shirobokov, Sergey and Kozlovskii, Borislav and Ruiz, Francisco JR and Mehrabian, Abbas and others},
  journal={arXiv preprint arXiv:2506.13131},
  year={2025}
}

@inproceedings{hu2025automated,
  title={Automated design of agentic systems},
  author={Hu, Shengran and Lu, Cong and Clune, Jeff},
  booktitle={International Conference on Learning Representations},
  volume={2025},
  pages={21344--21377},
  year={2025}
}

@inproceedings{wan2025rasc,
  title={Reasoning Aware Self-Consistency: Leveraging Reasoning Paths for Efficient {LLM} Sampling},
  author={Wan, Guangya and Wu, Yuqi and Chen, Jie and Li, Sheng},
  booktitle={Proceedings of the 2025 Conference of the Nations of the Americas Chapter of the Association for Computational Linguistics: Human Language Technologies (Volume 1: Long Papers)},
  year={2025},
  pages={3613--3635},
  publisher={Association for Computational Linguistics},
  doi={10.18653/v1/2025.naacl-long.184},
  url={https://aclanthology.org/2025.naacl-long.184/}
}

@article{shinn2023reflexion,
  title={Reflexion: Language agents with verbal reinforcement learning},
  author={Shinn, Noah and Cassano, Federico and Gopinath, Ashwin and Narasimhan, Karthik and Yao, Shunyu},
  journal={Advances in neural information processing systems},
  volume={36},
  pages={8634--8652},
  year={2023}
}

@inproceedings{besta2024got,
  title={Graph of thoughts: Solving elaborate problems with large language models},
  author={Besta, Maciej and Blach, Nils and Kubicek, Ales and Gerstenberger, Robert and Podstawski, Michal and Gianinazzi, Lukas and Gajda, Joanna and Lehmann, Tomasz and Niewiadomski, Hubert and Nyczyk, Piotr and others},
  booktitle={Proceedings of the AAAI conference on artificial intelligence},
  volume={38},
  number={16},
  pages={17682--17690},
  year={2024}
}

@inproceedings{agrawal2026gepa,
  title={Gepa: Reflective prompt evolution can outperform reinforcement learning},
  author={Agrawal, Lakshya A and Tan, Shangyin and Soylu, Dilara and Ziems, Noah and Khare, Rishi and Opsahl-Ong, Krista and Singhvi, Arnav and Shandilya, Herumb and Ryan, Michael J and Jiang, Meng and others},
  booktitle={International Conference on Learning Representations},
  volume={2026},
  pages={8479--8565},
  year={2026}
}

@inproceedings{zhao2024expel,
  title={Expel: Llm agents are experiential learners},
  author={Zhao, Andrew and Huang, Daniel and Xu, Quentin and Lin, Matthieu and Liu, Yong-Jin and Huang, Gao},
  booktitle={Proceedings of the AAAI Conference on Artificial Intelligence},
  volume={38},
  number={17},
  pages={19632--19642},
  year={2024}
}

@inproceedings{ouyang2026reasoningbank,
  title={Reasoningbank: Scaling agent self-evolving with reasoning memory},
  author={Ouyang, Siru and Yan, Jun and Hsu, I and Chen, Yanfei and Jiang, Ke and Wang, Zifeng and Han, Rujun and Le, Long and Daruki, Samira and Tang, Xiangru and others},
  booktitle={International Conference on Learning Representations},
  volume={2026},
  pages={94327--94354},
  year={2026}
}

@article{wang2023voyager,
  title={Voyager: An open-ended embodied agent with large language models},
  author={Wang, Guanzhi and Xie, Yuqi and Jiang, Yunfan and Mandlekar, Ajay and Xiao, Chaowei and Zhu, Yuke and Fan, Linxi and Anandkumar, Anima},
  journal={arXiv preprint arXiv:2305.16291},
  year={2023}
}
